\documentclass[sigconf]{acmart}
\renewcommand\footnotetextcopyrightpermission[1]{} % removes footnote with conference
\setcopyright{none}
\usepackage{amsfonts}
\usepackage{array, makecell}
\usepackage{mathtools}
\usepackage{graphicx}
\usepackage{subfigure}
\usepackage{enumerate}
\usepackage{amsmath}
\usepackage{color}
\usepackage{amsthm}
\usepackage{algorithm}
\usepackage{algpseudocode}
\usepackage{stfloats}
\usepackage{multirow}
\theoremstyle{definition}
\algtext*{EndWhile}% Remove "end while" text
\algtext*{EndIf}% Remove "end if" text
\algtext*{EndFor}% Remove "end for" text
\DeclareMathOperator*{\argmin}{arg\,min}

\newcommand{\bp}{\begin{proof} \small }
\newcommand{\ep}{\end{proof} \normalsize}
\newcommand{\epx}{\end{proof} \small}
\newcommand{\bpa}{\begin{proofappx} \footnotesize }
\newcommand{\epa}{\end{proofappx} \small }
\newtheorem{theorem}{Theorem}

\newtheorem{lemma}{Lemma}

\newtheorem*{theorem*}{Theorem}
\newtheorem*{proposition*}{Proposition}
\newtheorem*{corollary*}{Corollary}
\newtheorem*{lemma*}{Lemma}
\newtheorem*{assumption*}{Assumption}
\newtheorem*{definition*}{Definition}
\newtheorem*{claim*}{Claim}

\newcommand{\bm}[1]{\mbox{\boldmath $#1$}}

\newcommand{\be}{\begin{equation}}
\newcommand{\ee}{\end{equation}}
\newcommand{\bs}{\begin{subequations}}
\newcommand{\es}{\end{subequations}}
\newcommand{\bq}{\begin{eqnarray}}
\newcommand{\eq}{\end{eqnarray}}
\newcommand{\bqn}{\begin{eqnarray*}}
\newcommand{\eqn}{\end{eqnarray*}}

\newcommand{\ba}{\left[ \begin{array}}
\newcommand{\ea}{\\ \end{array} \right]}
\newcommand{\ben}{\begin{enumerate}}
\newcommand{\een}{\end{enumerate}}

\def\A{{\boldsymbol{A}}}

\def\b{{\boldsymbol{b}}}

\def\p{{\boldsymbol{p}}}

\def\x{{\boldsymbol{x}}}

\def\real{{\mathchoice%
{\hbox{\rm\setbox1=\hbox{I}\copy1\kern-.45\wd1 R}}
{\hbox{\rm\setbox1=\hbox{I}\copy1\kern-.45\wd1 R}}
{\hbox{\scriptsize\rm\setbox1=\hbox{I}\copy1\kern-.45\wd1 R}}
{\hbox{\scriptsize\rm\setbox1=\hbox{I}\copy1\kern-.45\wd1 R}}}}

\def\Zint{{\mathchoice{\setbox1=\hbox{\sf Z}\copy1\kern-.75\wd1\box1}
{\setbox1=\hbox{\sf Z}\copy1\kern-.75\wd1\box1}
{\setbox1=\hbox{\scriptsize\sf Z}\copy1\kern-.75\wd1\box1}
{\setbox1=\hbox{\scriptsize\sf Z}\copy1\kern-.75\wd1\box1}}}
\newcommand{\complex}{ \hbox{\rm C\kern-0.45em\rule[.07em]{.02em}{.58em}%
\kern 0.43em}}

\makeatletter
\newcommand{\algmargin}{\the\ALG@thistlm}
\makeatother
\newlength{\whilewidth}
\algdef{SE}[parWHILE]{parWhile}{EndparWhile}[1]
{\parbox[t]{\dimexpr\linewidth-\algmargin}{%
		\hangindent\whilewidth\strut\algorithmicwhile\ #1\ \algorithmicdo\strut}}{\algorithmicend\ \algorithmicwhile}%
\algnewcommand{\parState}[1]{\State%
	\parbox[t]{\dimexpr\linewidth-\algmargin}{\strut #1\strut}}

\ifodd 1

\else

\fi

\begin{document}
% paper title
% can use linebreaks \\ within to get better formatting as desired
\title{Autodidactic Neurosurgeon: Collaborative Deep Inference for Mobile Edge Intelligence via Online Learning}

\author{Letian Zhang}
%\authornote{Both authors contributed equally to this research.}
\email{lxz437@miami.edu}
%\orcid{1234-5678-9012}
\affiliation{%
	\institution{University of Miami}
	%\streetaddress{P.O. Box 1212}
	\city{Coral Gables}
	\state{Florida}
	\country{USA}
	\postcode{33146}
}

\author{Lixing Chen}
%\authornote{Both authors contributed equally to this research.}
\email{lx.chen@miami.edu}
%\orcid{1234-5678-9012}
\affiliation{%
	\institution{University of Miami}
	%\streetaddress{P.O. Box 1212}
	\city{Coral Gables}
	\state{Florida}
	\country{USA}
	\postcode{33146}
}

\author{Jie Xu}
%\authornote{Both authors contributed equally to this research.}
\email{jiexu@miami.edu}
%\orcid{1234-5678-9012}
\affiliation{%
	\institution{University of Miami}
	%\streetaddress{P.O. Box 1212}
	\city{Coral Gables}
	\state{Florida}
	\country{USA}
	\postcode{33146}
}

\begin{abstract}
    Recent breakthroughs in deep learning (DL) have led to the emergence of many intelligent mobile applications and services, but in the meanwhile also pose unprecedented computing challenges on resource-constrained mobile devices. This paper builds a collaborative deep inference system between a resource-constrained mobile device and a powerful edge server, aiming at joining the power of both on-device processing and computation offloading. The basic idea of this system is to partition a deep neural network (DNN) into a front-end part running on the mobile device and a back-end part running on the edge server, with the key challenge being how to locate the optimal partition point to minimize the end-to-end inference delay. Unlike existing efforts on DNN partitioning that rely heavily on a dedicated offline profiling stage to search for the optimal partition point, our system has a built-in online learning module, called Autodidactic Neurosurgeon (ANS), to automatically learn the optimal partition point on-the-fly. Therefore, ANS is able to closely follow the changes of the system environment by generating new knowledge for adaptive decision making. The core of ANS is a novel contextual bandit learning algorithm, called $\mu$LinUCB, which not only has provable theoretical learning performance guarantee but also is ultra-lightweight for easy real-world implementation. We implement our system on a video stream object detection testbed to validate the design of ANS and evaluate its performance. The experiments show that ANS significantly outperforms state-of-the-art benchmarks in terms of tracking system changes and reducing the end-to-end inference delay. 

\end{abstract}

% \begin{CCSXML}
% <ccs2012>
%   <concept>
%       <concept_id>10010147.10010257</concept_id>
%       <concept_desc>Computing methodologies~Machine learning</concept_desc>
%       <concept_significance>500</concept_significance>
%       </concept>
%   <concept>
%       <concept_id>10010147.10010178.10010224</concept_id>
%       <concept_desc>Computing methodologies~Computer vision</concept_desc>
%       <concept_significance>500</concept_significance>
%       </concept>
%   <concept>
%       <concept_id>10003120.10003138.10003140</concept_id>
%       <concept_desc>Human-centered computing~Ubiquitous and mobile computing systems and tools</concept_desc>
%       <concept_significance>500</concept_significance>
%       </concept>
%  </ccs2012>
% \end{CCSXML}

% \ccsdesc[500]{Computing methodologies~Machine learning}
% \ccsdesc[500]{Computing methodologies~Computer vision}
% \ccsdesc[500]{Human-centered computing~Ubiquitous and mobile computing systems and tools}

\maketitle

\section{Introduction}
Deep neural networks (DNNs) have been the state-of-the-art solution in recent years for many functionalities routinely integrated in mobile devices, e.g., face recognition and speech assistant. However, efficiently integrating current and future deep learning (DL) breakthrough within resource constrained mobile devices is challenging. Although steps have been taken recently to enable DL functionalities on mobile devices, e.g., model compression \cite{xie2019source}, lightweight machine learning libraries \cite{tfLite} and new-generation hardware \cite{sima2018apple}, they are unlikely to be a \emph{one-size-fits-all} solution that can address the immediate needs of all existing mobile devices due to the substantial heterogeneity in terms of their computing capacity. A recent study by Facebook \cite{wu2019machine} shows that over 50\% mobile devices are using processors at least six years old, limiting what is possible of Facebook AI service. Also, for wearable mobile devices, e.g., smart wristbands, their computing resource limitation is not due to the temporary technical deficiency but design requirements to guarantee portability \cite{mao2016dynamic}. Therefore, external booster becomes necessary to realize the full potential of DNN on mobile devices.

Current wisdom focuses on the Multi-Access Edge Computing (MEC) \cite{mao2017survey}, a new paradigm and key technology of 5G that moves cloud-like functionality towards edge servers close to data sources. The idea is to configure DNNs on edge servers to which the input data will be sent from mobile devices on the occurrence of inference requests. While recognizing the advantages of edge computing for DNN inference, previous empirical studies \cite{li2018edge} reveal that its performance is highly sensitive to the bandwidth between edge servers and mobile devices. For massive input like video streaming, the delay of DNN inference by offloading the entire input data to the edge server can become on par with or even worse than that on the local mobile device. With the observation that the data size of some intermediate results (i.e., the output of intermediate layers in DNN) is significantly smaller than that of the raw input data, collaborative deep inference between the mobile device and the edge server starts to attract increasing attention recently as it is able to leverage the power of both on-device processing and computation offloading. The idea is to partition the DNN into a front-end part running on the mobile device and a back-end part running on the edge server. The mobile device executes the DNN model up to an intermediate layer, and then sends the smaller intermediate results to the edge server for computing the rest part of the DNN. Compared to either pure on-device processing or computation offloading to an edge server, collaborative deep inference is expected to be more reliable and flexible in balancing the transmission and computation workload between the mobile device and the edge server, and hence has the potential of optimizing the end-to-end inference performance.

\vspace{-10 pt}

\subsection{Numerical Insights}
To demonstrate the effectiveness of collaborative deep inference, Fig. \ref{fig:totalPartitionLatency} shows the end-to-end inference delay, when Vgg16 is partitioned at different layers under an uplink transmission speed 12 Mbps. As can be seen, partitioning Vgg16 at the fc1 layer reduces the end-to-end inference delay by $29.64\%$ compared to on-device processing or edge offloading.

Apparently, partitioning a DNN does not always outperform on-device processing and edge offloading, and the optimal partition point depends on many factors, among which, the  computing capability of the edge server and the network condition are major. To illustrate the impact of computing capability of the edge server on the optimal DNN partition point, Fig. \ref{fig:edgeComputationDiff} shows the end-to-end inference delay if the DNN is partitioned at different layers for a high-capability edge server (i.e., GPU and low workload) and a low-capability edge server (i.e., CPU and high workload). As can be seen, the optimal partition point tends to be later (or even pure on-device processing in this case) as the benefit of offloading to a lower-capability edge server is smaller. The other major factor is the network condition, which affects the transmission delay. Fig. \ref{fig:networkSpeedDiff} shows the optimal partition points under three network conditions: High uplink rate (50 Mbps), Medium uplink rate (16 Mbps), and Low uplink rate (4 Mbps). As can be seen, a lower uplink rate tends to push the partition point later. However, because the output data size is not necessarily monotonically increasing/decreasing as we move to later layers, the optimal partition point is complexly dependent on the uplink transmission rate. To summarize, the computing capability of the edge server and the network condition critically affect the collaborative deep inference performance.

\begin{figure*}[tt]
    \centering
    \begin{minipage}[t]{0.31\textwidth}
        \centering
        \includegraphics[width=\textwidth]{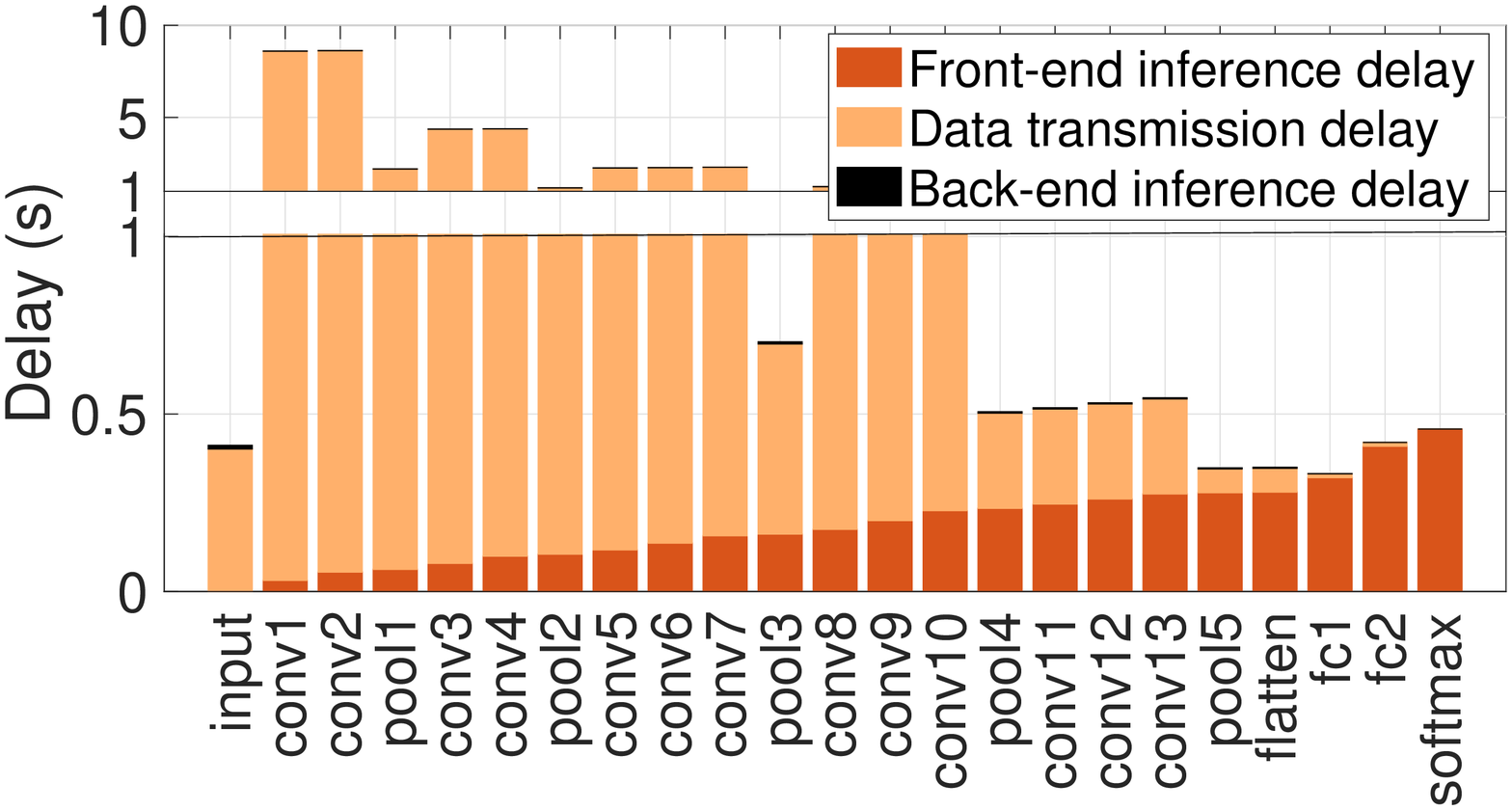}
        \caption{End-to-end delay: front-end inference delay + transmission delay + back-end inference delay. (Vgg16)}
        \label{fig:totalPartitionLatency}
    \end{minipage}
    \hspace{0.02\textwidth}
    \begin{minipage}[t]{0.31\textwidth}
        \centering
        \includegraphics[width=\textwidth]{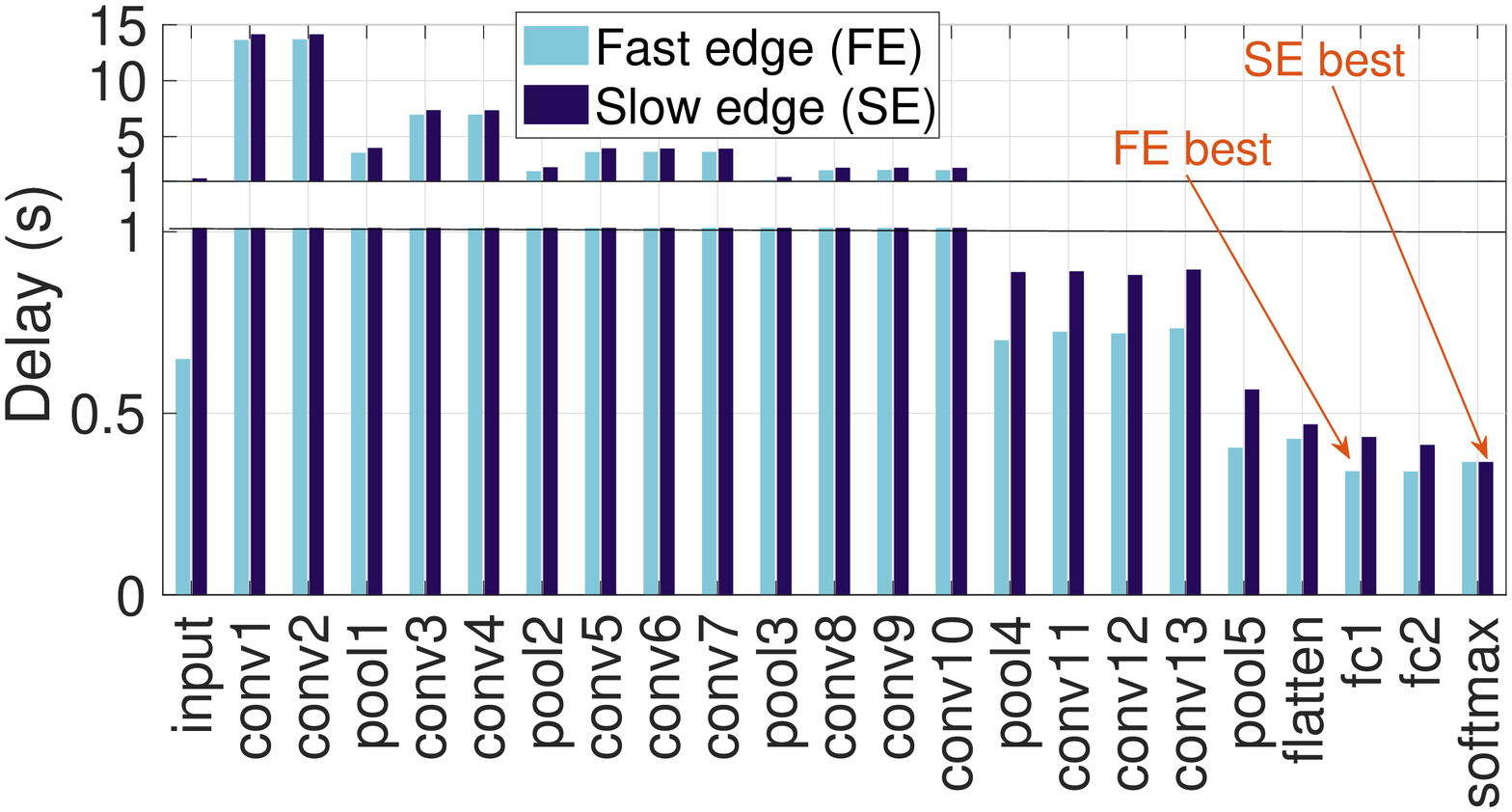}
        \caption{End-to-end delay at different partition points under different edge capabilities. (Vgg16)}
        \label{fig:edgeComputationDiff}
    \end{minipage}
    \hspace{0.02\textwidth}
    \begin{minipage}[t]{0.31\textwidth}
        \centering
        \includegraphics[width=\textwidth]{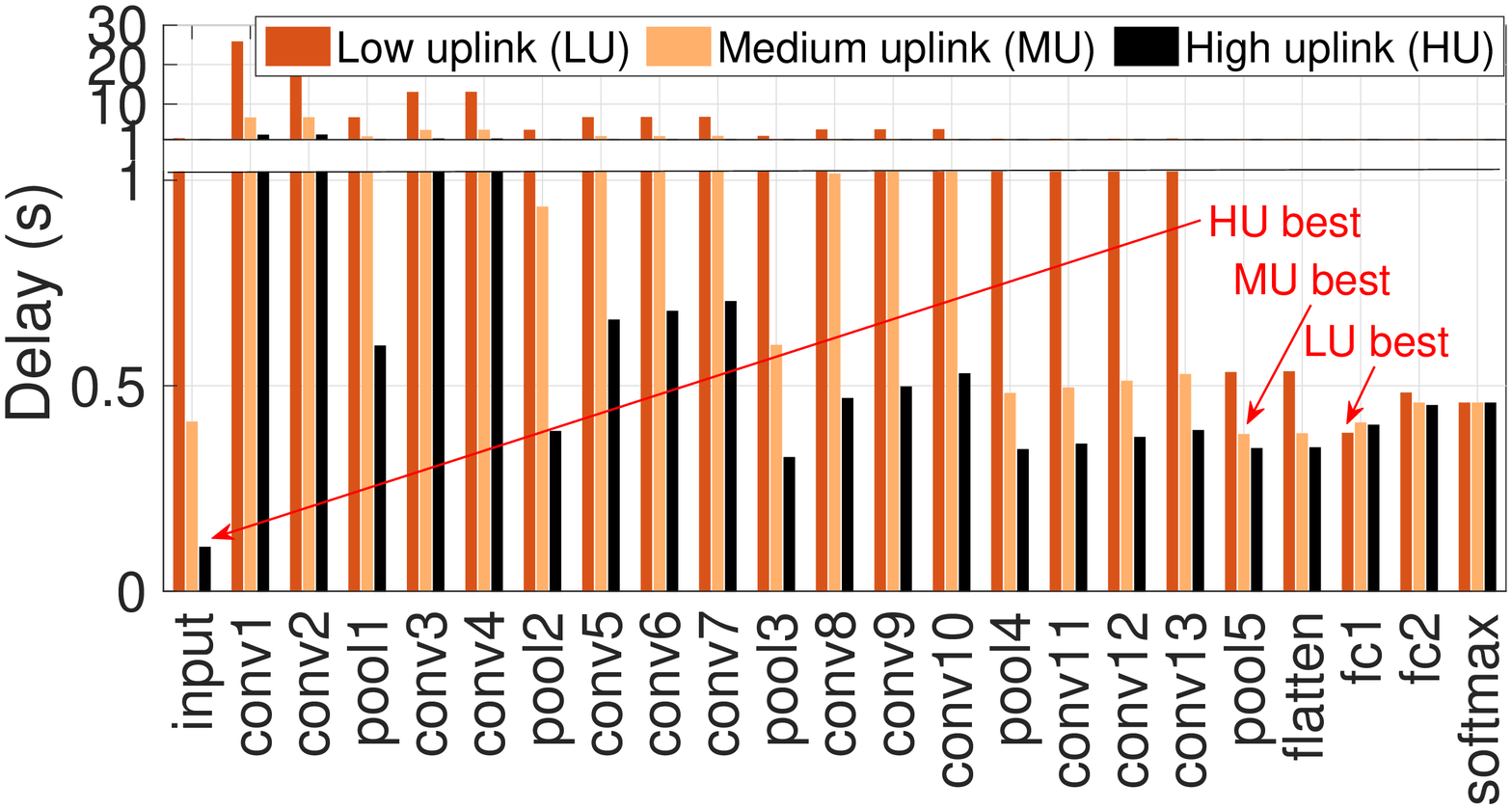}
        \caption{End-to-end delay at different partition points under different network conditions. (Vgg16)}
        \label{fig:networkSpeedDiff}
    \end{minipage}
    \vspace{-10 pt}
\end{figure*}

\vspace{-10 pt}

\subsection{Why Online Learning?}
    The crux of efficient collaborative deep inference is locating the optimal partition point of the DNN to distribute workload between the mobile device and the edge server. Existing efforts rely heavily on the offline profiling of the layer-wise DNN inference delay as a function of the system parameters, e.g., the uplink data transfer rate and the mobile device/edge server processing speed \cite{kang2017neurosurgeon, li2019edge, hu2019dynamic, eshratifar2019jointdnn}. With this offline-acquired knowledge, the partition point can be determined via solving an easy optimization problem. This method can be further extended to online adaptation using real-time input of system parameters. However, there are several major drawbacks of such an offline profiling approach. 
	
    \textbf{Adaptation to New Environment}: The knowledge acquired during offline profiling can be easily outdated considering the highly dynamic and uncertain environment. For example, the network uplink rate can change due to the dynamic spectrum management of the wireless carrier, the multi-user interference, and the mobility of mobile devices; the edge server processing capability may also change over time due to the edge server resource management to support multi-tenancy or even the change of edge servers themselves due to location change. Once the offline knowledge becomes outdated, its suggestion can lead to arbitrarily bad results. While performing offline profiling whenever a new environment is encountered is possible, it incurs significant overhead to generate accurate predictions.
	
	\textbf{Limited Feedback}: Existing offline profiling methods, even with periodic updating, require explicit real-time system parameters as input, e.g., the uplink data transfer rate and the workload on edge servers. These parameters, however, not only are ever-changing, but also can be very difficult for an end-user mobile device to obtain in practice, if not impossible. Often the case, the mobile device can observe only the overall delay between sending the data and receiving the inference result from the edge server, but is unable to accurately decompose this delay into different components (e.g., transmission delay and processing delay). This limited feedback challenge is similar to the congestion control problem in the classic Transport Control Protocol (TCP), where the end-user adjusts its congestion window based on only a binary congestion signal from the network as a summary of all network effects.
	
	\textbf{Layer Dependency}: Existing offline profiling methods adopt a layer-wise approach, which profiles the inference delay of each individual DNN layer depending on the system parameters. Clearly, profiling can be laborious for very deep neural networks as layers become many. More importantly, the layer-wise approach has an inherent drawback since it neglects the interdependency between layers. In fact, the overall inference delay is not even a simple sum of per-layer delay due to the inter-layer optimization performed by state-of-the-art DNN software packages, e.g., cuDNN \cite{chetlur2014cudnn}, which has a non-negligible impact on the total inference time. 

All these dynamics and uncertainties presented in the collaborative deep inference system and their a priori unknown impacts on the inference performance call for an online learning approach that can learn to locate the optimal partition point on-the-fly.

\vspace{-10 pt}

\subsection{Our Contribution}
In this paper, we design and build a collaborative deep inference system for video stream object detection, which contains a novel online learning module, called Autodidactic Neurosurgeon (ANS), to automatically learn the optimal partition point based on limited inference delay feedback. Object detection and tracking in video streams is a core component for many emerging intelligent applications and services, e.g., augmented reality, mobile navigation and autonomous driving. The mobile device in our system continuously receives video frames captured by an on-device camera, and selects, for each frame (or a small batch of video frames), a partition point to perform collaborative deep inference for object detection with the edge server. ANS has several salient features:

(1) ANS avoids the large overhead incurred in the laborious offline profiling stage. Instead, it learns the optimal partition on-the-fly and hence easily adapts to new environments.

(2) ANS does not need hard-to-acquire system parameters as input. Instead, it utilizes only the limited delay feedback of past collaborative inference tasks. 

(3) ANS exploits the intrinsic dependency of DNN layers without the need to learn each possible partition point individually, thereby tremendously accelerating the learning speed.

(4) ANS explicitly handles key frames captured in the video stream by assigning higher priority to those frames, thus providing differentiated service to frames during learning.

(5) ANS requires ultra-lightweight computation and minimal storage and hence, it is easy to deploy in practical systems. 

(6) The core of ANS is a novel online learning algorithm developed under the contextual bandit framework, and it has a provable performance guarantee. More technical innovations will be discussed later in Section \ref{sec:ANS}.

We highlight that as on-device processing is a special partition decision, ANS complements existing efforts such as DNN model compression on pushing DL intelligence into mobile devices, while providing added benefits by exploiting multi-access edge computing. We implement the collaborative deep inference system on a hardware testbed, where a Nvidia Jetson TX2 device, a fair representation of mobile devices, wirelessly connects to a GPU-powered edge server (Dell Alienware Desktop). Experiment results show that ANS is able to accurately learn the optimal partition point and hence accelerates deep inference for various DNN model structures under various wireless network settings. 
		
\section{System Architecture} \label{sec:system_model}
In this section, we describe the architecture of the collaborative deep inference system. A pictorial overview is given in Fig. \ref{fig:systemOverview}. 

\begin{figure}[!h]
    \centering
    \includegraphics[width=0.85\linewidth]{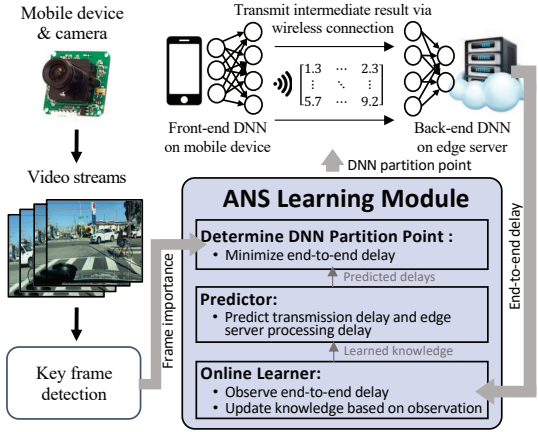}
    \caption{An overview of the system architecture.}
    \label{fig:systemOverview}
    \vspace{-10 pt}
\end{figure}

\vspace{-10 pt}
\subsection{Deep Neural Network Partition}
We first formalize DNN partitioning, discuss its impact on the end-to-end inference delay and introduce preliminaries. 
   	
\subsubsection*{\textbf{Marking Partition Points}} 
   			
Let $\mathcal{P} = \{0, 1, 2, \dots, P\}$ collect all potential DNN partition points. A partition point $p\in\mathcal{P}$ partitions a DNN into two parts: 1) the front-end part, $\texttt{DNN}^\texttt{front}_p$, contains layers from the input to the partition point $p\in\mathcal{P}$, and 2) the back-end part, $\texttt{DNN}^\texttt{back}_p$, contains layers from the partition point $p$ to the output layer. For example, if the partition point is placed at $p = 2$, then $\texttt{DNN}^\texttt{front}_p$ contains layers $\{1,2\}$ and $\texttt{DNN}^\texttt{back}_p$ contains layers $\{3, 4, \dots, P\}$. The partition points $p = 0$ and $p = P$ are the special cases: the partition $p = 0$ gives an empty $\texttt{DNN}^\texttt{front}_p$ which means the mobile device transmits raw input data to the edge server to run the entire DNN; the partition $p = P$ gives an empty $\texttt{DNN}^\texttt{back}_p$ indicating that all DNN layers are executed on the mobile device. The output of $\texttt{DNN}^\texttt{front}_p$ is called the intermediate result of partition $p$, denoted by $\psi_p$. Note that the intermediate output will be sent to the edge server for further processing, and we assume that the intermediate output $\psi_p$ includes necessary overhead for data packet transmission (e.g., packet header) and follow-up DNN merging (e.g., information about the partition point).

\subsubsection*{\textbf{Breakdown of DNN Inference Delay}}	
The end-to-end collaborative deep inference delay consists of three main parts: 
(1) \textbf{Front-end inference delay} $d^\texttt{f}_p$ of $\texttt{DNN}^\texttt{front}_p$ on the mobile device; (2) \textbf{Transmission delay} $d^\texttt{tx}_p$ for transmitting the intermediate output $\psi_p$ from the mobile device to the edge server; (3) \textbf{Back-end inference delay} $d^\text{b}_p$ of $\texttt{DNN}^\texttt{back}_p$ on the edge server.

The data size of the final inference result is usually small and hence the transmission delay for the final result return is neglected for the ease of problem formulation. The end-to-end inference delay with partition point $p$ is therefore $d_{p} := d_p^\texttt{f}  + d_p^\texttt{tx} + d_p^\texttt{b} + \eta$, where $\eta$ is a Gaussian random variable to model the randomness in the inference and transmission processes. 
   			
The transmission delay $d^\texttt{tx}_p$ is determined by the data size of the intermediate result $\psi_p$ and the wireless uplink transmission rate, which varies depending on the network condition. The inference delays $d^\texttt{f}_p$, $d^\texttt{b}_p$ of partitioned DNNs depend on many more factors: the number of DNN layers, the computational complexity of component layers, the inter-layer optimization, and also the processing speed of the mobile device/edge server. While some of them are fixed once the DNN structure is given (e.g., the number of layers and layer-wise computational complexity), others depend on the configuration of the computing platform (e.g., inter-layer optimization tools) and may also be time-varying (e.g., multi-user scheduling by the edge server). 
   		
We note that the configuration of the computing platform on the mobile device is relatively stable and fully revealed to the decision maker, i.e., the mobile device itself, and hence the front-end inference delay $d^\texttt{f}_p$ of $\texttt{DNN}^\texttt{front}_p$ can be easily measured statistically for a given DNN using methods similar to offline profiling. In the experiment, we use the application-specific profiling method in \cite{eshratifar2019jointdnn} to obtain the expected inference delay of $\texttt{DNN}^\texttt{front}_p$. Compared with the layer-wise statistical modeling method adopted in \cite{kang2017neurosurgeon,li2019edge,hu2019dynamic}, this method provides more accurate estimations because it takes into account the inter-layer optimization. Now, the key difficulty lies in learning $d^\texttt{tx}_p + d^\texttt{b}_p$ for different partition points as a result of unknown and time-varying edge computing capability and network condition. For ease of exposition, we define $d^\texttt{e}_p = d^\texttt{tx}_p + d^\texttt{b}_p$ as the \textbf{edge offloading delay}. 

\vspace{-10 pt}
\subsection{Edge Offloading Delay Prediction}
\label{sec:edgeLinearPrediction}
To obtain $d^\texttt{e}_p$, our idea is to learn a prediction model that maps contextual features of a partition point to the edge offloading delay. Since  learning works online, this prediction model updates itself using the limited feedback information about the past observed $d^\texttt{e}_p$ to closely follow the changes in the (unknown) system parameters. Using contextual features of partition points has a clear advantage over learning the delay performance of individual partition points separately, especially when the number of possible partition points is large. This is because the underlying relationship between different partition points is captured by their contextual features, and hence, knowledge gained by choosing one particular partition point can be easily transferred to learning about the performance of all other unselected partition points. 
	
\subsubsection*{\textbf{Constructing Contextual Features of Partitions}}
We first construct contextual features associated with $\texttt{DNN}^\texttt{back}_p$ that may affect $d^\text{b}_p$. Intuitively, the back-end inference delay is linearly dependent on the computation complexity of the back-end partition $\texttt{DNN}^\texttt{back}_p$, which usually is captured by the the number of \emph{multiply-accumulate}
(MAC) units contained in $\texttt{DNN}^\texttt{back}_p$. However, our experiment shows that the required computation time for one MAC unit is in fact different for different types of DNN layers. This is because different DNN layers allow different levels of parallel computation, especially when GPU is involved in the computation process. Since different partition points result in different combinations of layer types in  $\texttt{DNN}^\texttt{back}_p$, simply using the total number of MAC units to predict $d^\text{b}_p$ is problematic. To address this issue, instead of using a single scalar value for the total number of MAC units, we calculate the number of MAC units for each layer type, and use this vector for learning the inference delay. Specifically, we consider three main types of layers in DNN:  \emph{i}) convolutional layer, \emph{ii}) fully-connected layer, \emph{iii}) and activation layer, and count the number of MAC units in layers of different types, denoted by $m_p^\texttt{c}$, $m_p^\texttt{f}$ and $m_p^\texttt{a}$, respectively, for a given partition point $p$. In addition, we also count the number of convolutional layers $n_p^\texttt{c}$, fully-connected layers $n_p^\texttt{f}$, and activation layers $n_p^\texttt{a}$ in $\texttt{DNN}^\texttt{back}_p$. These numbers will affect inter-layer optimization and hence are also useful for learning the overall inference delay. 
	    
For the transmission delay $d^\texttt{tx}_p$, although the wireless uplink rate may be unknown, it is still clear that $d^\texttt{tx}_p$ linearly depends on the data size of the intermediate output $\psi_p$ of the front-end partition $\texttt{DNN}^\texttt{front}_p$. 
	    
In sum, the contextual feature of a partition point $p$ is collected in $\bm{x}_p = [ m_p^\texttt{c}, m_p^\texttt{f}, m_p^\texttt{a}, n_p^\texttt{c}, n_p^\texttt{f}, n_p^\texttt{a}, \psi_p ] ^\top$. Here, we slightly abuse notation to use $\psi_p$ to denote the data size of the intermediate results. 

In Fig. \ref{fig:contextualFeatureIllustration}, we provide an example to illustrate the contextual features of a particular partition point. 
	 
\subsubsection*{\textbf{Linear Prediction Model}}
Although the best model for predicting the edge offloading delay is unclear due to the obscured process of DNN inference, we adopt a linear model due to the reasons mentioned above. In addition, compared to other more complex and non-linear prediction models (such as a neural network), the linear model is much simpler and requires minimal resource on the mobile device. We show later in the experiments that this linear model is in fact validated to be a very good approximation.  
	    
Specifically, our prediction model has the form $d^\texttt{e}_p = \theta^\top \bm{x}_p$, where $\theta$ is the linear coefficients to be learned, which captures the unknown effects of the unknown system parameters (i.e., wireless uplink condition, computation capability of the edge server) on the delay performance. In runtime, the coefficients will be updated online as new observations of $d^\texttt{e}_p$ as a result of the partition decision $p$ are obtained. How to update these coefficients will be explained later in Section \ref{sec:ANS}. 
	    
Note that there is a practical reason why we learn $d^\texttt{e}_p$ as a whole rather than $d^\texttt{tx}_p$ and $d^\texttt{b}_p$ individually. As an end-user mobile device, it can observe $d^\texttt{e}_p$, by calculating the difference between the time when the data is sent and the time when the result is received for a selected partition point $p$. However, often it is very difficult for the mobile device to decompose this feedback into $d^\texttt{tx}_p$ and $d^\texttt{b}_p$ unless additional information is provided by the edge server. In this paper, we focus on this more challenging limited feedback scenario, although the individual feedback can also be easily incorporated into our framework.

\vspace{-10 pt}
\subsection{Object Detection in Video Stream}
Now, we explain how collaborative deep inference works in the context of video stream object detection. In video stream object detection, the mobile device uses its camera to capture a video and aims to in real-time detect objects in the successive frames of the video by feeding them one-by-one to a pre-trained DNN. The flow of frames is indexed by $\mathcal{T} = \{1,2,\dots,T\}$ and for each frame $t \in \mathcal{T}$, the mobile device has to pick a partition point $p_t$ to perform collaborative deep inference with the edge server. Note that pure on-device processing and pure edge offloading are special cases by choosing $p_t = P$ and $p_t = 0$, respectively. Once the inference is done, detection results (i.e., object bounding boxes and class labels) are displayed on the video. To assist online learning, the mobile device also records the actual edge offloading delay $d^\texttt{e}_p$ unless it chooses pure on-device processing (i.e. $p_t = P$). 
		
Suppose the linear coefficients $\theta$ are already learned, then the mobile device should pick a partition point to minimize the DNN inference delay by solving the following problem:
\begin{align}
    p_t := \argmin\nolimits_{p\in\mathcal{P}} ~ d^\texttt{f}_p +  \bm{\theta}^\top\bm{x}_p
\end{align}
where the first term is the front-end inference delay and the second term is the edge offloading delay. However, since the coefficients $\theta$ are \textit{a priori} unknown, the mobile device has to try different partition points and use the observed edge offloading delay feedback to form a good estimate of $\theta$. Clearly, there is subtle \textit{exploitation v.s. exploration tradeoff}, i.e., whether the mobile device should pick the partition point that solves the above minimization problem based on the current estimate of $\theta$ or pick other possible partition points to form a more accurate estimate of $\theta$ for future use. 
		
This exploitation v.s. exploration tradeoff is further complicated by \textit{key frames} in the video stream. Key frames are the most representative frames in video streams, which contain main elements or important events (e.g., entrance of new objects or scene change). It is often the case that these key frames have higher requirements on the inference performance, e.g., lower inference delay requirement. To provide differentiated quality of services to key and non-key frames, frames must be treated differently during online learning: while non-key frames may tolerate a larger inference delay as a result of exploring different partition points, key frames should be handled with more care using the best-known partition point as much as possible. 
		
Since key frame detection itself is not the focus or the main contribution of this paper, we apply one of the most widely-used key frame detection methods, namely structural similarity (SSIM) \cite{wang2004image}, to determine key frames. Fig. \ref{fig:keyFrameIllustration} illustrates the idea of SSIM. 
        
\begin{figure*}[tt]
    \centering
    \begin{minipage}[t]{0.34\textwidth}
        \centering
        \includegraphics[width=\textwidth]{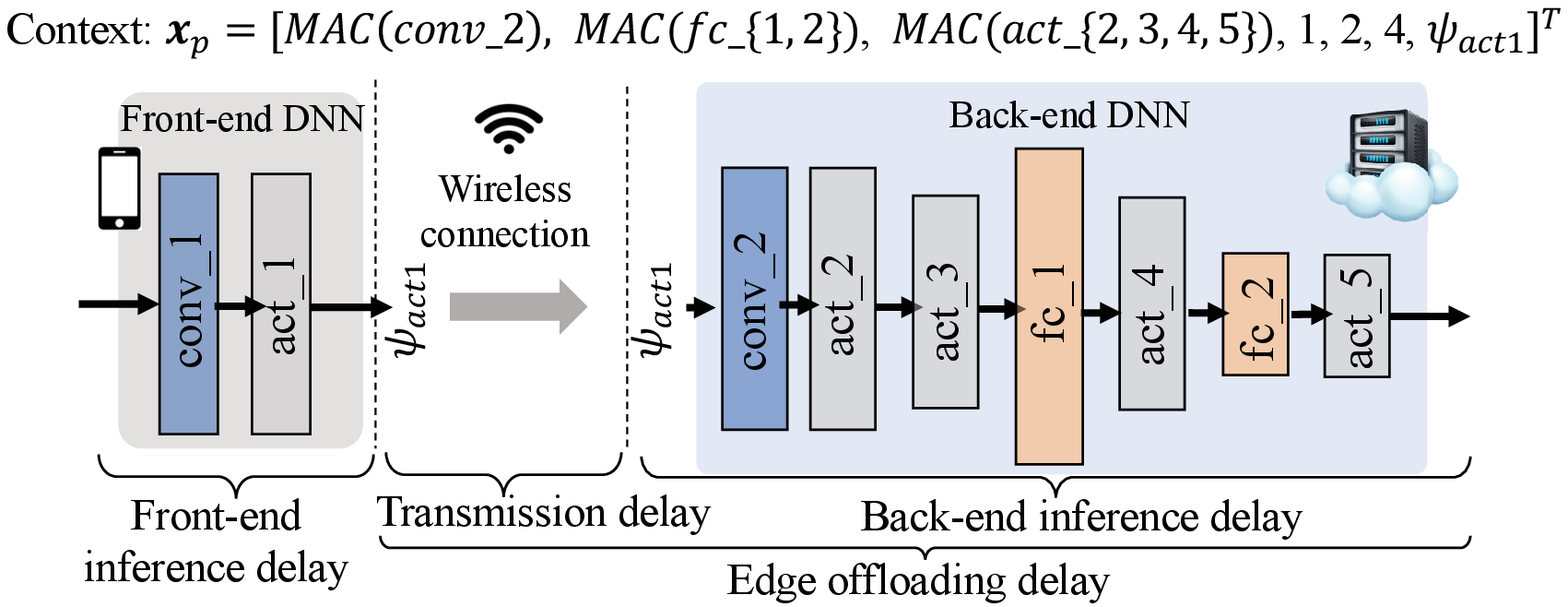}
        \caption{Contextual feature vector for a DNN partition point at \texttt{act\_1}.}
        \label{fig:contextualFeatureIllustration}
    \end{minipage}
    \hspace{0.01\textwidth}
    \begin{minipage}[t]{0.31\textwidth}
        \centering
        \includegraphics[width=\textwidth]{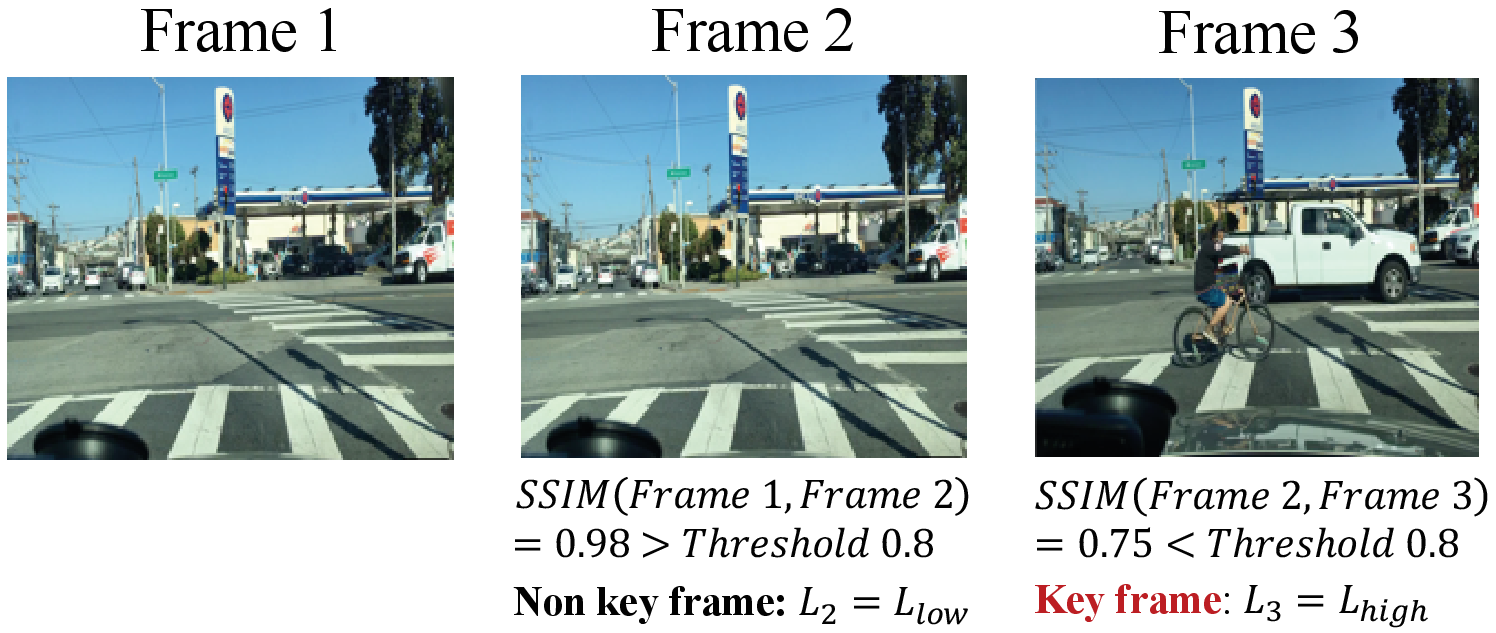}
        \caption{Key frame detection using SSIM. A key frame is detected if it is sufficiently different from the previous frame.}
        \label{fig:keyFrameIllustration}
    \end{minipage}
    \hspace{0.01\textwidth}
    \begin{minipage}[t]{0.31\textwidth}
        \centering
        \includegraphics[width=\textwidth]{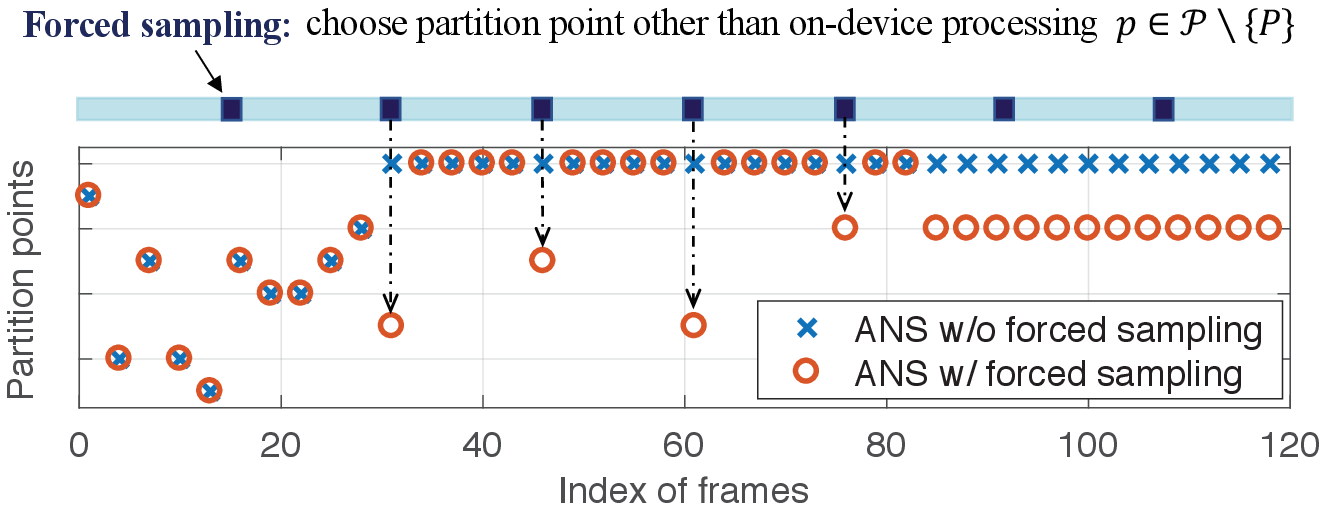}
        \caption{Forced sampling: forced sampling is activated only when the partition decision were to be on-device processing.}
        \label{fig:forcedSamplingIllustration}
    \end{minipage}
    \vspace{-0.2 in}
\end{figure*}

\section{Autodidactic Neurosurgeon}
\label{sec:ANS}
In this section, we describe the design of the online learning module, called Autodidactic NeuroSurgeon (ANS), in our collaborative deep inference system. The core of ANS is an online learning algorithm that can predict the inference delay of different partition points and base on the prediction to select partition points. Since we adopt a linear prediction model as explained in Section \ref{sec:edgeLinearPrediction}, LinUCB \cite{chu2011contextual}, a classic online learning algorithm for linear models that gracefully handles the exploitation v.s. exploration tradeoff, seems a good candidate for solving our problem. However, there are two unique challenges for LinUCB to work effectively in our system (which will be explained later). Therefore, a new online learning algorithm, called $\mu$LinUCB, is developed to support ANS. In what follows, we first explain how LinUCB works and its limitation in ANS. Next, we propose $\mu$LinUCB, prove its theoretical performance guarantee and analyze its complexity.

\vspace{-10 pt}
\subsection{LinUCB and its Limitation}
The basic idea of LinUCB is an online linear regression algorithm, which incrementally updates the linear coefficients using newly acquired feedback. However, when making decisions, LinUCB takes into account the confidence of the prediction for different actions' expected payoff (i.e., the delay of different partition points in our case). Put in the context of DNN partition, LinUCB maintains two auxiliary variables $\A \in \mathbb{R}^{d\times d}$ and $\b \in \mathbb{R}^{d\times 1}$ for estimating the coefficients $\bm{\theta}$. For each video frame $t$, $\bm{\theta}$ is estimated by $\hat{\bm{\theta}}_t = \A^{-1}_{t-1}\b_{t-1}$, and the partition point for frame $t$ is selected to be
\begin{align}\label{eq:linucb_opt}
	p_t = \argmin_{\p\in\mathcal{P}}~~d^\texttt{f}_p + \hat{\bm{\theta}}^\top\x_p - \alpha\sqrt{\x^\top_p \A^{-1}_{t-1}\x_p}
\end{align}	    
In the function to be minimized, the first term $d^\texttt{f}_p$ is the front-end inference delay of partition point $p$, which is assumed to be known; the second term $\hat{\bm{\theta}}^\top\x_p$ is the predicted edge offloading delay of partition point $p$ using the current estimate $\hat{\bm{\theta}}$; the last term $\alpha\sqrt{\x^\top_p \A^{-1}_{t-1}\x_p}$ represents the confidence interval of the edge offloading delay prediction. A larger confidence interval indicates that the prediction is not accurate enough and hence, even if the predicted delay of a partition point $p$ is low, the chance to select this partition point should be decreased. After the inference request is completed and the realized edge offloading delay $d_{p_t}^\texttt{e}$ is observed, the auxilary variables are incrementally updated as $\A_t \gets \A_{t-1} + \x_{p_t} \x^\top_{p_t}$ and $\b_t \gets \b_{t-1} + \x_{p_t}d_{p_t}^\texttt{e}$.

However, LinUCB has two major limitations for it to work effectively in ANS:
	    
\underline{\textbf{Limitation \#1}}: LinUCB treats each frame equally for the learning purpose. In other words, being a key frame or not does not affect the way LinUCB selects a partition point and hence, key frames can also experience high inference delay because of  unlucky bad choices of partition points due to exploration. 
	    
\underline{\textbf{Limitation \#2}}: This limitation is in fact detrimental. Among all possible partition points, the partition point $p = P$, or pure on-device processing, is actually a very special partition point that does \textit{not} follow the linear prediction model. This is because the edge offloading delay is always 0 once $p = P$ is selected and \textit{any} linear coefficient is a ``correct'' coefficient since the contextual feature associated with $p = P$ is a zero vector. If, for some video frames, $p = P$ is selected by LinUCB for deep inference, then the auxiliary variables $A_t$ and $b_t$ do not get updated since there is no feedback/new information about the edge offloading delay. As a result, LinUCB will select $p = P$ according to the selection rule \eqref{eq:linucb_opt} for the next video frame and thereafter, essentially being forced to stop learning and trapped in pure on-device processing for all future video frames. Therefore, LinUCB fails to work in ANS.

\vspace{-10 pt}
\subsection{$\mu$LinUCB} \label{mu_algorithm}
In light of the limitations of LinUCB, we propose a new online learning algorithm, called $\mu$LinUCB, by providing mitigation mechanisms to LinUCB in order to support ANS. As we will see, these mitigation mechanisms are quite intuitive. However, rather than being heuristic, they grant $\mu$LinUCB  provable performance guarantee with a careful choice of algorithm parameters.

\underline{\textbf{Mitigation \#1}}: To provide differentiated inference service to key and non-key frames, ANS assigns weights to frames and incorporates these weights when selecting partition points. Specifically, each frame $t$ is assigned with a weight $L_t$ depending on whether it is a key frame or not (or the likelihood of being a key frame), and the partition point selection rule is modified to be
\begin{align}\label{eq:mulinucb_opt}
    p_t = \argmin_{\p\in\mathcal{P}}~~d^\texttt{f}_p + \hat{\bm{\theta}}^\top\x_p - \alpha\sqrt{(1-L_t)\x^\top_p \A^{-1}_{t-1}\x_p}
\end{align}           
As key frames will receive a larger weight, the confidence interval (i.e., the third term) plays a smaller role when ANS selects a partition point. Therefore, ANS tends to play safe with key frames by exploiting partition points that are so far learned to be good. 
        
\underline{\textbf{Mitigation \#2}}: To escape from being trapped in pure on-device processing, a natural idea is to add randomness in partition point selection. Because partition points other than the pure on-device processing have a chance to be selected, new knowledge about the edge offloading delay and hence $\bm{\theta}$ can be acquired. Our implementation of this randomness idea is through a \textit{forced sampling} technique. Specifically, for a total number of $T$ video frames, we define a forced sampling sequence $\mathcal{F} = \{t | t = nT^\mu, t \le T, n = 1,2,\dots \}$, where $\mu$ is a design parameter. If the index $t$ of a video frame belongs to $\mathcal{F}$, then $\mu$LinUCB forces ANS to sample a partition point other than $p = P$. In other words, $p = P$ is not an option for these frames. According to the design of the sequence, forced sampling occurs every $T^\mu$ frames. Note that, forced sampling has no effect on frames when $p = P$ is not the selected partition point if the classic LinUCB were applied. Fig. \ref{fig:forcedSamplingIllustration} illustrates the idea of forced sampling.

The pseudocode of $\mu$LinUCB is given in Algorithm \ref{alg:mulinucb}. It follows the same idea of estimating $\bm{\theta}$ using an online linear regresser as LinUCB. The key difference is that it incorporates the frame importance weights and forced sampling when making partition point selection decisions.

\begin{algorithm}
	\caption{ASN with $\mu$LinUCB algorithm}
	\begin{algorithmic}[1]
		\State Construct context $x_p$ for candidate partition points $\forall p \in \mathcal{P}$				
		\State Obtain front-end inference delay estimate $d^\texttt{f}_p, \forall p \in \mathcal{P}$
		\State Determine forced sampling sequence $\mathcal{F}$ .				
		\State Initialize $\A_0 = \beta I_d$, $b_0 = \boldsymbol{0}$.
				
		\For {each frame $t = 1, 2, \cdots, T$}
			\State Detect key frame and assign weight $L_t$
			\State  Compute current estimate $\hat{\bm{\theta}}_t = \A_{t-1}^{-1} b_{t-1}$.
				\For {each candidate partition point $p \in \mathcal{P}$}
				    \State Compute $\hat{d}^\texttt{e}_p = \hat{\bm{\theta}}^\top_t\x_p - \alpha\sqrt{(1-L_t)\x^\top_p \A^{-1}_{t-1}\x_p}$.
				\EndFor
				\If {$t \in \mathcal{F}$} {\color{red} \Comment{Forced sampling}}
				    \State Choose $p_t = \arg \min_{p \in \mathcal{P}_{{\{\ne P\}}}} d^\texttt{f}_p + \hat{d}^\texttt{e}_p$.
				\Else
                    \State Choose $p_t = \arg \min_{p \in \mathcal{P}} d^\texttt{f}_p + \hat{d}^\texttt{e}_p$.
			    \EndIf
			    \If {$p_t \ne P$} {\color{red}\Comment{Not pure on-device processing}}
			        \State Observe $d^\texttt{e}_{p_t}$ once inference is done.
			    	\State $\A_t \gets \A_{t-1} + \x_{p_t} \x^\top_{p_t},~~\b_t \gets  \b_{t-1} + \x_{p_t}d^\texttt{e}_{p_t}$.
			    \Else
			        \State $\A_t = \A_{t-1},~~ \b_t = \b_{t-1}$.
			    \EndIf
		\EndFor
	\end{algorithmic}
    \label{alg:mulinucb}
\end{algorithm}

\subsubsection*{\textbf{Theoretical Performance Guarantee}}
The parameter $\mu$ is a critical parameter of $\mu$LinUCB (hence the name), which controls the frequency of forced sampling. There is indeed a crucial tradeoff that determines the performance of ANS. Since forced sampling always selects a suboptimal partition point when $p = P$ (i.e., pure on-device processing) is indeed the best option, a smaller $\mu$ (i.e., more frequent forced sampling) results in increased averaged inference delay. With a larger $\mu$ (i.e., less frequent forced sampling), ANS can be trapped in pure on-device processing for a long time if $p = P$ is indeed a suboptimal option, again leading to increased average inference delay. In the theorem below, we characterize what is a good choice of $\mu$ and the resulting performance of $\mu$LinUCB. 
        
\begin{theorem}\label{theo:regertbound} 
Under mild technical assumptions, the regret (i.e., the delay performance difference compared to an oracle algorithm that selects the optimal partition point for all $T$ frames) of $\mu$LinUCB, denoted by $R(T)$, satisfies: $\forall \delta \in (0, 1)$, with probability at least $1-\delta$, $R(T)$ can be upper bounded by 
\begin{equation}\label{regretBound}
    \max\{O(T^{0.5+\mu}\log(T/\delta)), O(T^{1-\mu})\}
\end{equation}
\end{theorem}
\begin{proof}
	The proof is given in Appendix \ref{sec:proofOfBound}.
\end{proof}

According to Theorem \ref{theo:regertbound}, by choosing $\mu \in (0, 0.5)$, the regret bound is sublinear in $T$, implying that the average end-to-end inference delay asymptotically achieves the best possible end-to-end inference delay when $T \to \infty$. For a finite $T$, this bound also gives a characterization of the convergence speed of $\mu$LinUCB. In addition, by choosing $\mu = 0.25$, the order of the regret bound is minimized at $O(T^{0.75}\log(T))$. 
\begin{figure}[t]
    \includegraphics[width=0.85\linewidth]{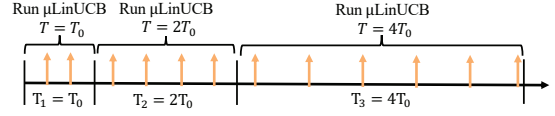}
    \caption{Forced sampling frequency decreases as phase length increases.}
    \label{fig:doublingTrickIllustration}
    \vspace{-20 pt}
\end{figure}

\subsubsection*{\textbf{Handling Unknown $T$}}
$\mu$LinUCB requires knowing the number $T$ of video frames for object detection to determine the frequency of forced sampling. This is clearly an ideal scenario and largely does not hold in practice. To handle the scenario when $T$ is unknown, $\mu$LinUCB can be modified as follows: $\mu$LinUCB starts with a large frequency of forced sampling and gradually reduces the frequency as more video frames have been analyzed. Gradually reducing the forced sampling frequency is reasonable because as more video frames have passed, ANS has obtained a more accurate estimate of $\theta$. Therefore, if pure on-device processing, i.e., $p = P$, is selected by ANS after many video frames, it is very likely that pure on-device processing is indeed the best inference option and hence, less forced sampling should be performed to reduce the unnecessary overhead. To give a concrete example of this strategy, we can divide the sequence of video frames into phases, indexed by $i = 1, 2, ...$. Each phase $i$ contains $T_i = \lfloor 2^i \cdot T_0\rfloor_{ i = \{1, 2, \cdots \}}$ video frames where $T_0 \in \mathbb{N}^+$ is an integer constant. Within each phase, ANS runs $\mu$LinUCB with a known number of video frames $T_i$. As $T_i$ is doubled every phase, the forced sampling interval, namely $T_i^\mu$, also increases. This means that the forced sampling frequency decreases. For this particular strategy, it can  still be proved that $\mu$LinUCB achieves a sublinear regret bound when $\mu \in (0, 0.5)$ even if $T$ is unknown. Fig. \ref{fig:doublingTrickIllustration} illustrates this increasingly sparse forced sampling sequence. 

\subsubsection*{\textbf{Complexity Analysis}}
For a DNN with $P$ possible partition points and a contextual feature vector of size $d$, we analyze the time and space complexity of $\mu$LinUCB for each frame $t$ as follows. \textit{Time Complexity}. Inversing the matrix $\A_t$ incurs a time complexity $O(d^3)$ \cite{boley1987survey}. Obtaining $\hat{\bm{\theta}}_t$ incurs a time complexity $O(d^2)$. Computing $\hat{d}^\texttt{e}_p$ for every possible partition point $p$ has a complexity $O(d^2 + 2d)$ and hence, the total complexity is $O(P(d^2 + 2d))$. Comparing $\hat{d}^\texttt{e}_p$ to obtain the optimal partition point has a complexity $O(P)$. Finally, updating $\A_t$ and $\b_t$ has a complexity $O(d^2 + d)$. Therefore, the total time complexity is $\max\{O(d^3), O((P+2)d^2\}$. Since $d$ is usually small (in our implementation $d = 7$), the overall time complexity for each frame is linear in the number of partition points. \textit{Space Complexity}. $\mu$LinUCB needs to keep variables $\A_{t-1}$, $\b_{t-1}$, $x_p$, $d^\texttt{f}_p$ and a constant  indicating the forced sampling frequency in memory. To compute $\hat{d}^e_p$, additional temporary memory is needed, which has space complexity $O(d^2 + 2d + P)$. Overall, the space complexity is $O(2d^2 + (P+3)d + 2d)$. Again, because $d$ is small, the space complexity is linear in $P$. In sum, $\mu$LinUCB incurs negligible computation complexity compared to regular deep inference tasks as it involves only a small number of simple operations and requires keeping a small number of variables.

\vspace{-10 pt}

\section{Experiment Results}
\subsection{Implementation and Setup}
\textbf{Testbed.} We build a hardware testbed to validate the design of ANS and evaluate its performance in a collaborative deep inference system for video stream object detection. We use NVIDIA Jetson TX2 developer module as the mobile device. It is equipped with a NVIDIA Pascal GPU, a shared 8 GB 128 bit LPDDR4 memory between GPU and CPU and an on-board camera. A Dell Alienware workstation is employed as the edge server, which is equipped with Intel Core i7-8700K CPU@3.70GHZ$\times12$, two Nvidia GeForce GTX 1080 Ti GPUs, and 11 GB memory. The mobile device and edge server are wirelessly connected by  point-to-point Wi-Fi, and we use WonderShaper \cite{wondershaper} to set the wireless transmission speed to emulate different network conditions. 
		
\textbf{DL Models and Platforms.} Three state-of-the-art DNNs, namely Vgg16 \cite{simonyan2014very}, YoLo \cite{redmon2016you} and ResNet50 \cite{he2016deep} are considered in the experiment. We implement ANS on both TensorFlow and PyTorch, two popular machine learning platforms, and run deep inference on these DNN models and perform DNN partitioning. We use Netscope Analyzer \cite{netscope}, a web-based tool, for visualizing and analyzing DNN network architectures. For chain topology DNNs, we mark a partition point after each layer. However, it should be noted that some DNN models are not chain topology, in which case the residual block method \cite{eshratifar2019bottlenet} can be used to determine the partition points (e.g., ResNet50 has 16 concatenated residual blocks).
        
\textbf{Video Input and Detection Output.} The input video is captured by on-board camera of the mobile device using OpenCV. The video frames are captured with $1280 \times 720$ pixels and then resized to $640\times480$ pixels for screen display. For DNN inferences, these video frames are converted to image inputs of default size $\texttt{Width}\times\texttt{Height}\times\texttt{Channels}$: Vgg16  -- $224\times224\times3$; YoLo -- $416\times416\times3$; ResNet -- $224\times224\times3$. After the DNN processes the video frame, detection output is generated, which includes the bounding box locations of candidate objects, their predicted class labels and corresponding confidence. The bounding boxes and class labels are then displayed on a screen connected to the mobile device. 

\textbf{Benchmarks.} The following benchmarks are used in the evaluation of the online learning module ANS.

(1) Oracle: Oracle selects the optimal partition point for each frame. We obtain the Oracle decision by measuring the performance of all partition points for 100 times and then picking the partition point with the minimum average delay.

(2) Pure Edge Offloading (EO): The mobile device offloads all frames to the edge server. The frames are processed on edge server and the results are then returned to the mobile device.

(3) Pure On-device Processing (MO): All the frames are processed on the mobile device with no offloading.

(4) Neurosurgeon \cite{kang2017neurosurgeon}: It is an offline profiling approach for collaborative deep inference.

\subsection{Results and Discussions}
\begin{table}[t]
	\centering
	\caption{Prediction error of ANS and layer-wise method when the uplink rate is high/medium/low and the edge server uses GPU/CPU.}
	\begin{tabular}{|>{\raggedright}p{0.20\linewidth}|p{0.07\linewidth} p{0.07\linewidth} p{0.09\linewidth} |p{0.08\linewidth} p{0.08\linewidth} p{0.09\linewidth}|}
		\hline
		\multirow{2}{*}{}& \multicolumn{3}{c|}{ANS} & \multicolumn{3}{c|}{Layer-wise} \\
		\cline{2-7}
		& Vgg16  & YoLo & ResNet & Vgg16  & YoLo & ResNet\\
		\hline
		Low/GPU &  0.43\% & 0.52\% & 3.52\% & 9.87\%  & 9.90\% & 12.68\% \\
		Medium/GPU &  3.01\% & 3.11\% & 5.10\% & 19.33\% & 14.23\% & 21.12\% \\
		High/GPU &  3.06\% & 3.94\% & 9.97\% & 21.42\%  & 22.70\% & 51.58\% \\
		Low/CPU &  0.39\% & 0.67\% & 2.97\% & 8.97\% & 10.02\% & 13.94\% \\
		Medium/CPU &  2.98\% & 3.15\% & 4.78\% & 21.69\% & 16.32\% & 28.53\% \\
		High/CPU &  3.12\% & 3.36\% & 7.96\% & 25.76\% & 24.35\% & 49.61\% \\
		\hline
	\end{tabular}
	\label{estimationError}
\vspace{-10pt}
\end{table}

\subsubsection*{\textbf{Delay Prediction Error and Learning Convergence}}

\begin{figure}[h]
\begin{minipage}[t]{0.46\linewidth}
\includegraphics[width=\textwidth]{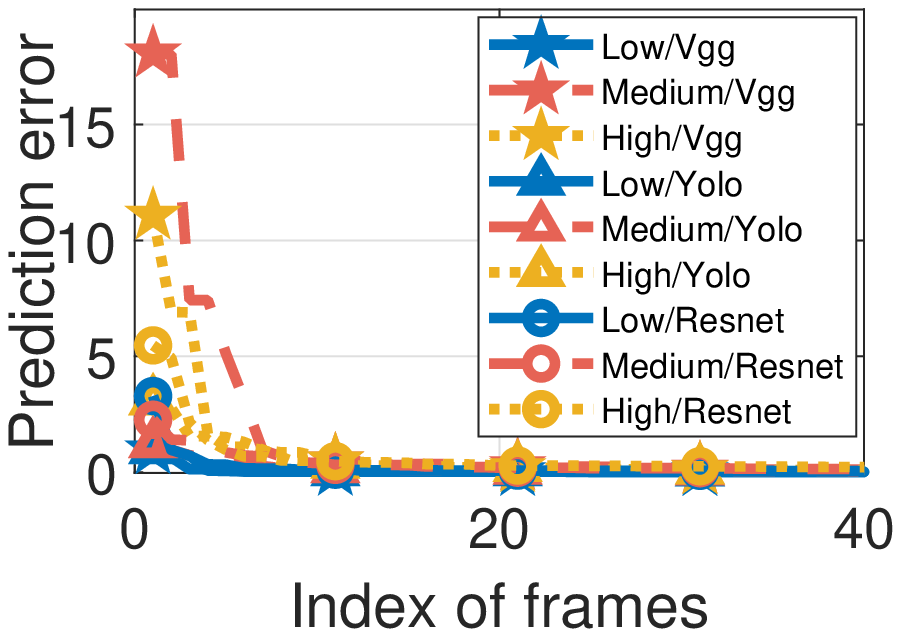}
\vspace{-10 pt}
\caption{Online prediction error reduces as more frames are analyzed.}
\label{fig:errorEvolution}
\end{minipage}
\hspace{0.03\linewidth}
\begin{minipage}[t]{0.46\linewidth}
\includegraphics[width=\textwidth]{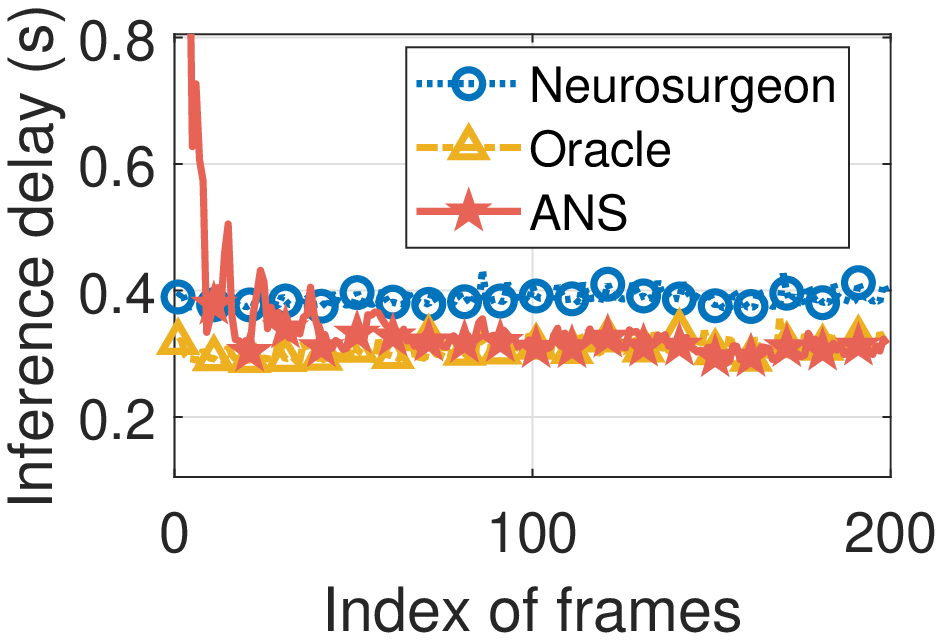}
\vspace{-10 pt}
\caption{End-to-end inference delay reduces as more frames are analyzed.}
\label{fig:compareMethods}
\end{minipage}
\vspace{-10 pt}
\end{figure}

Table \ref{estimationError} shows the edge offloading delay prediction error of ANS after 300 video frames, which is also compared to the layer-wise method used in \cite{hu2019dynamic, kang2017neurosurgeon, li2019edge} that neglects the inter-layer optimization. In all the tested environments (i.e., different combinations of network condition and edge computing capability), ANS achieves an excellent prediction performance, far outperforming the layer-wise method. The improvement is the most significant in high uplink rate scenarios because the back-end inference delay is dominant in the edge offloading delay and neglecting the impact of inter-layer optimization on the inference delay introduces significant errors. Fig. \ref{fig:errorEvolution} further shows how the prediction error evolves as more video frames have been analyzed for ANS. As can been seen, ANS learns very fast and can accurately predict delay in about 20 video frames. 
Fig. \ref{fig:compareMethods} shows the runtime average end-to-end inference delay achieved by ANS, Oracle and Neurosurgeon. The average delay of ANS quickly converges to that of Oracle in about 80 frames, starting from zero knowledge about the system environment. Compared to Neurosurgeon, both ANS and Oracle is better because they consider the inter-layer optimization during DNN inference while Neurosurgeon performs layer-wise profiling. In fact, this comparison is not fair to ANS as Neurosurgeon requires the real-time information about the system (i.e., the transmission rate and edge server workload) while ANS does not.

\subsubsection*{\textbf{End-to-End Inference Delay Improvement}}

\begin{figure*}[t]
    \subfigure[Vgg16]{
    % \label{NetworkImpactVgg}
    \includegraphics[width=0.24\textwidth]{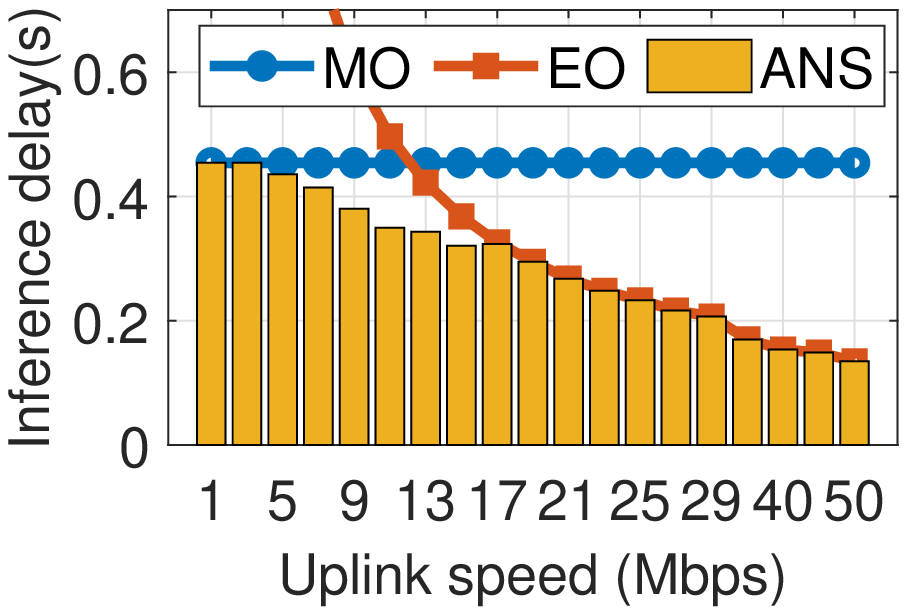}}
    \subfigure[YoLo]{
    % \label{NetworkImpactYolo}
    \includegraphics[width=0.24\textwidth]{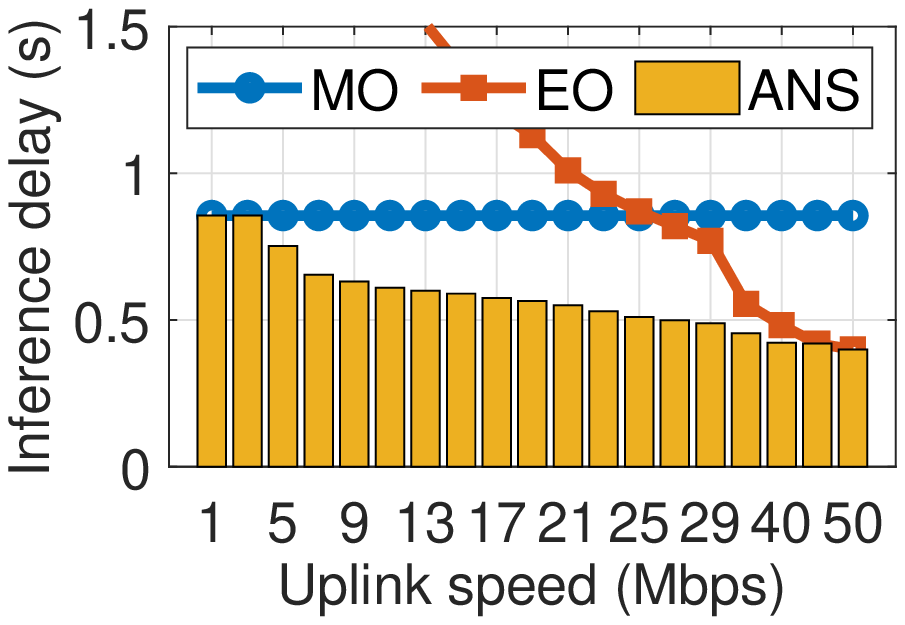}}
    \subfigure[ResNet]{
    % \label{NetworkImpactResNet}
    \includegraphics[width=0.24\textwidth]{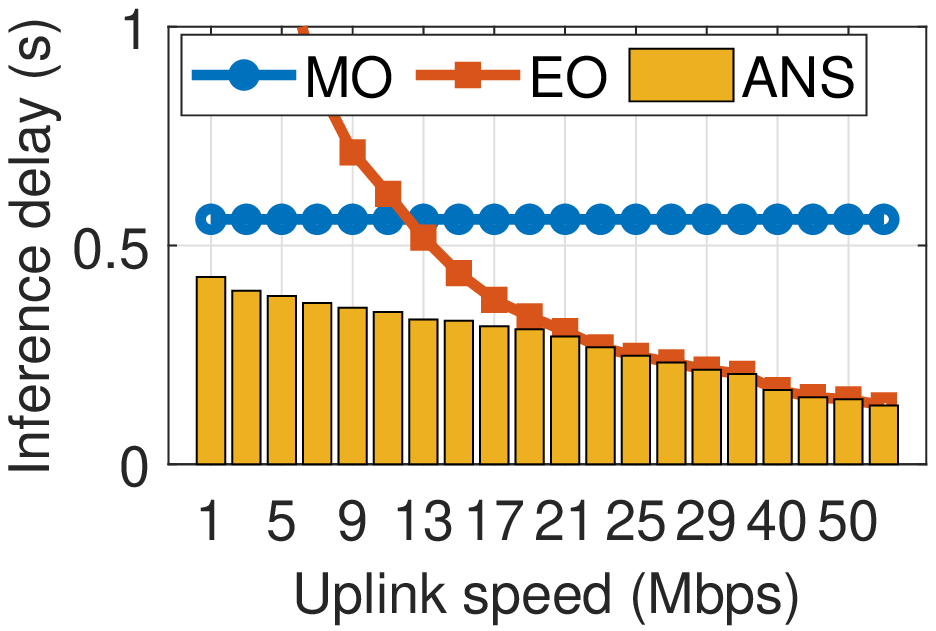}}
    \subfigure[Best Reduction]{
    \includegraphics[width=0.24\textwidth]{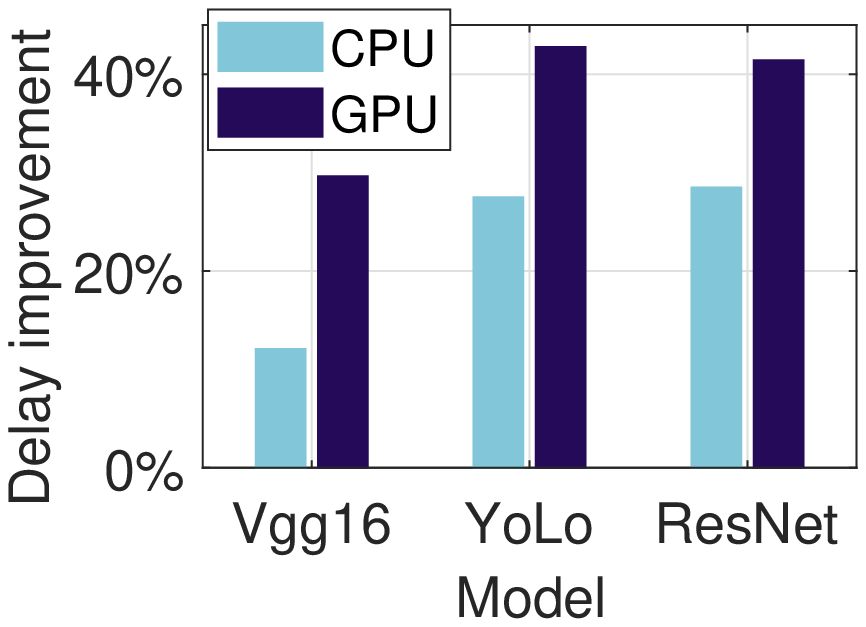}}
    \vspace{-0.2 in}
    \caption{End-to-end inference delay achieved by MO, EO and ANS, and delay reduction of ANS.}
    \label{fig:NetworkImpact}
    \vspace{-10 pt}
\end{figure*}

Fig. \ref{fig:NetworkImpact} shows the end-to-end inference delay achieved by MO, EO and ANS under different uplink transmission rates when the edge server uses GPU. When the transmission rate is low, the inference delay of ANS is close to MO. This is because ANS tends to run the entire DNN on the mobile device to avoid large transmission delay for sending data to the edge server. When the transmission rate is high, the inference delay of ANS is close to EO. This is because running the entire DNN on the edge server significantly reduces the inference delay at a small extra cost due to the small transmission delay. When the transmission rate is moderate, ANS is able to make an effective trade-off between on-device processing and edge offloading. Fig. \ref{fig:NetworkImpact}(d) summarizes the end-to-end delay improvement in the best cases when the edge server uses CPU or GPU for all three DNNs. As it suggests, collaborative inference using ANS achieves a larger improvement when the edge server is more powerful. 

\subsubsection*{\textbf{Adaption to Changing Environment}}
Fig. \ref{fig:dynamicvgg} shows how ANS can track the change of the environment and adapt its partition point when the network condition changes. The upper subplot shows the evolution of the uplink transmission rate over time. The middle subplot shows the partition points selected by ANS with $\mu$LinUCB. As a comparison, we also show in the bottom subplot the partition results of the classic LinUCB. At the beginning, the uplink transmission rate is high and hence, ANS partitions at the input layer, sending the raw input image to the edge server. At the 150th frame, the transmission rate changes to be small, which makes pure on-device processing the optimal choice. Although ANS does not directly observe this network condition change, it quickly adapts its partition decision to be pure on-device processing in about 20 frames (at about the 170th frame). Later at the 390th frame, the network condition improves, and ANS is able to adapt to the new optimal partition point (i.e., fc1) in about 80 frames (at about the 470th frame). Finally at the 630th frame, the network condition becomes the best again, and after about 50 frames (at about the 680th frame), ANS changes to pure edge offloading again. 

Note that during the second phase when the network condition is bad, forced sampling is activated, which periodically tries a partition point other than pure on-device processing. As a result, when the network condition improves, ANS is able to detect this change and adapt to the new optimal partition. By contrast, LinUCB gets stuck in the pure on-device processing since the 170th frame, thereby losing its learning ability thereafter. 

\begin{figure}[t]
\subfigure[Changing uplink rate]{
\includegraphics[width=0.85\linewidth]{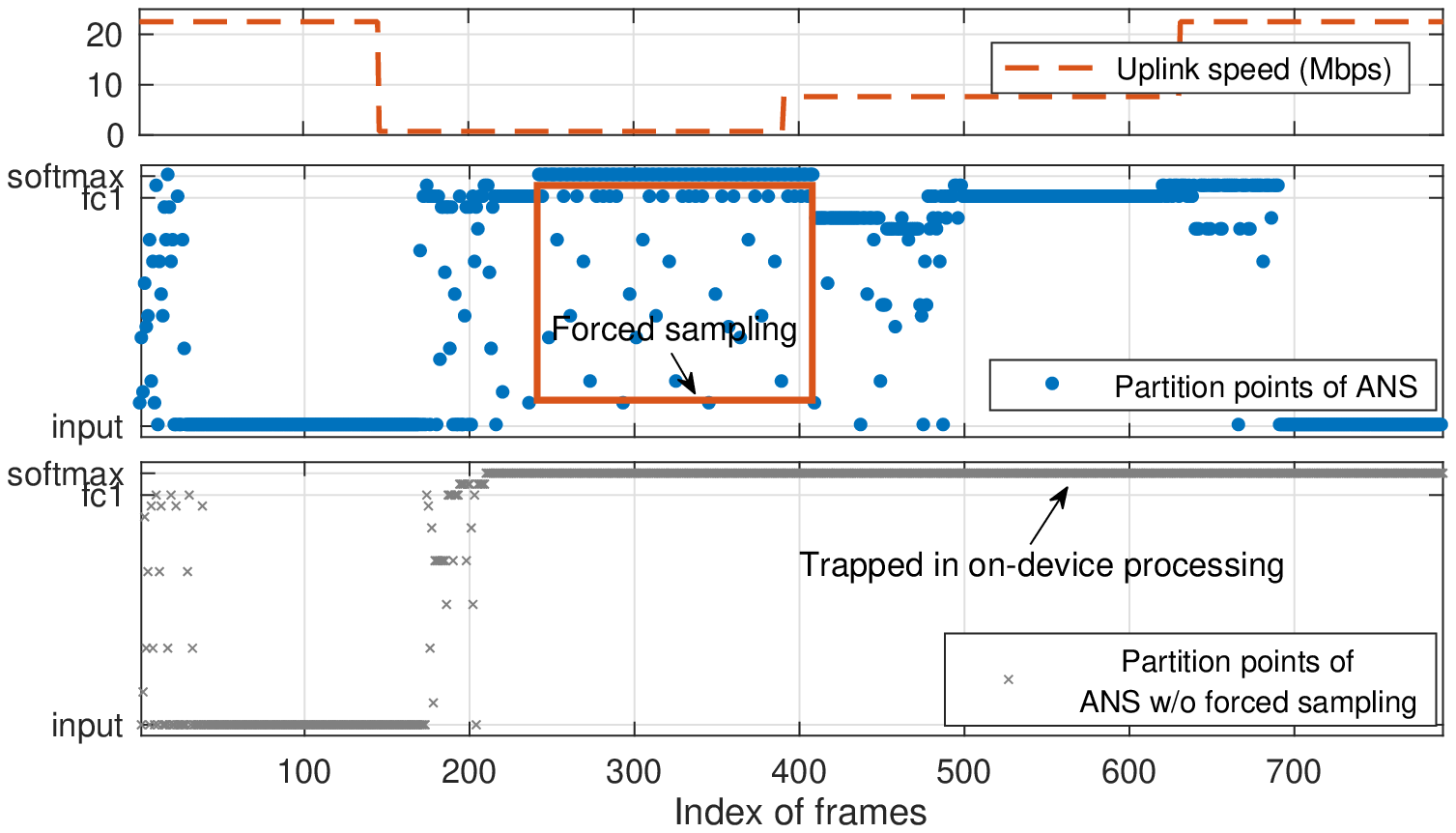}
\label{fig:dynamicvgg}}
\subfigure[Changing edge processing rate]{
\includegraphics[width=0.85\linewidth]{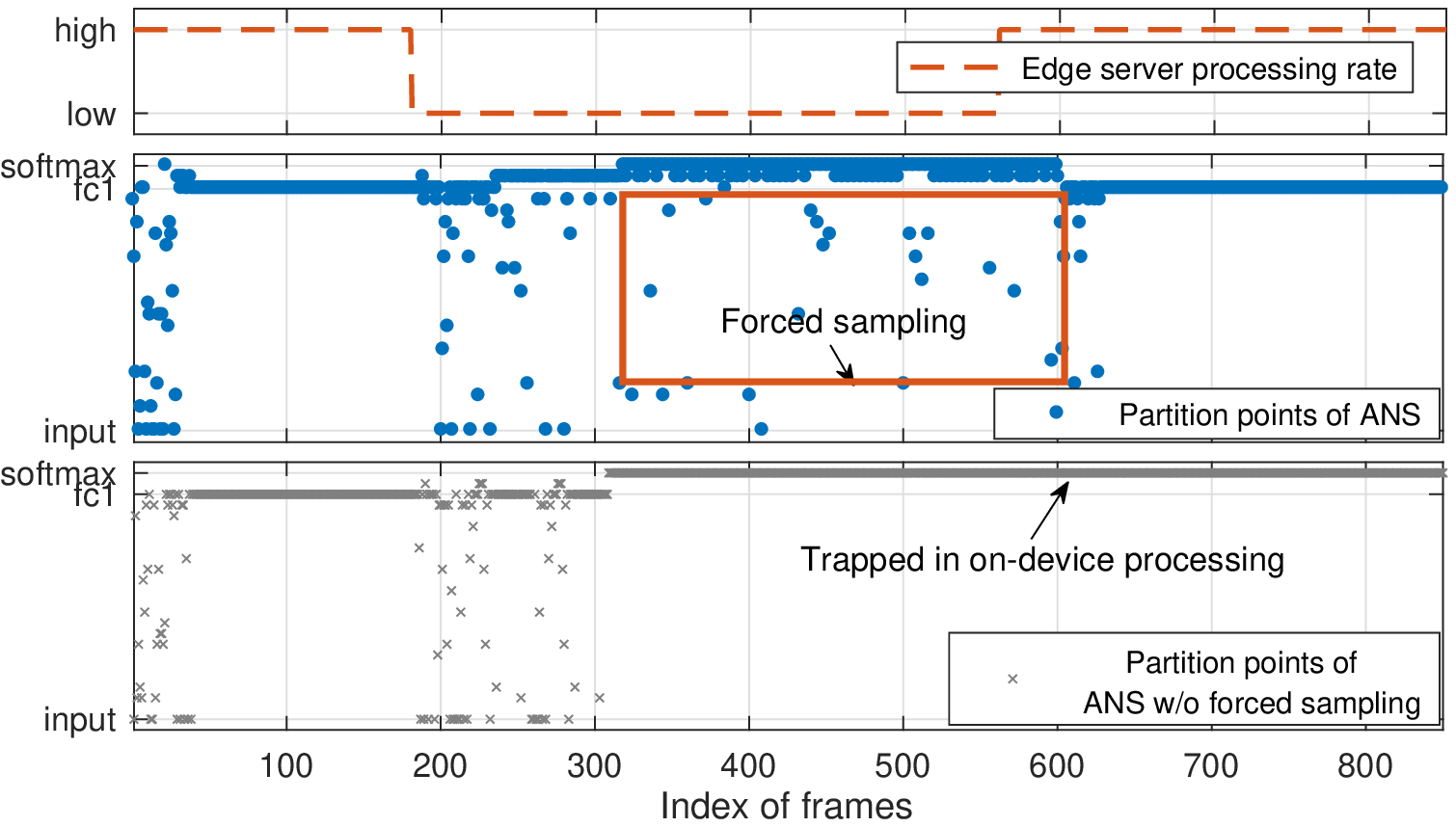}
\label{fig:dynamicEdge}}
\vspace{-10 pt}
\caption{ANS under changing environment.}
\vspace{-10 pt}
\end{figure}

Similar to Fig. \ref{fig:dynamicvgg}, Fig. \ref{fig:dynamicEdge} demonstrates the ability of ANS to adapt to the changing environment by varying the overall workload on the edge server over time. Similar observations can be made in this experiment. 

We further show how the frequency of environment change impacts ANS in Fig. \ref{fig:changeFrequency}. In the simulations, the wireless network is either in a fast (i.e., 50 Mbps) or a slow (i.e., 5 Mbps) state. For each frame, the network switches to the other state with probability $P_f$ and stays in the current state with probability $1-P_f$. As can been seen, ANS achieves excellent performance when the environment is relatively stable, namely when the switching probability is low. When the switching probability is high, ANS can become worse than MO due to the overhead incurred during learning. 

\begin{figure}[t]
\begin{minipage}[t]{0.46\linewidth}
\includegraphics[width=\textwidth]{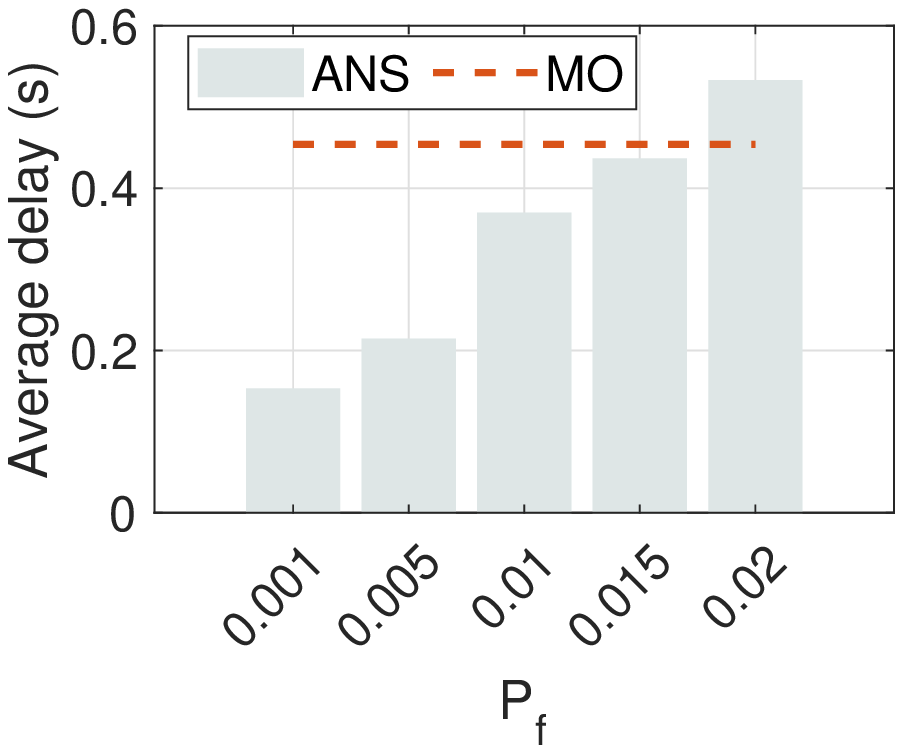}
\caption{Average inference delay at different environment change frequencies.}
\label{fig:changeFrequency}
\end{minipage}
\hspace{0.03\linewidth}
\begin{minipage}[t]{0.46\linewidth}
\includegraphics[width=\textwidth]{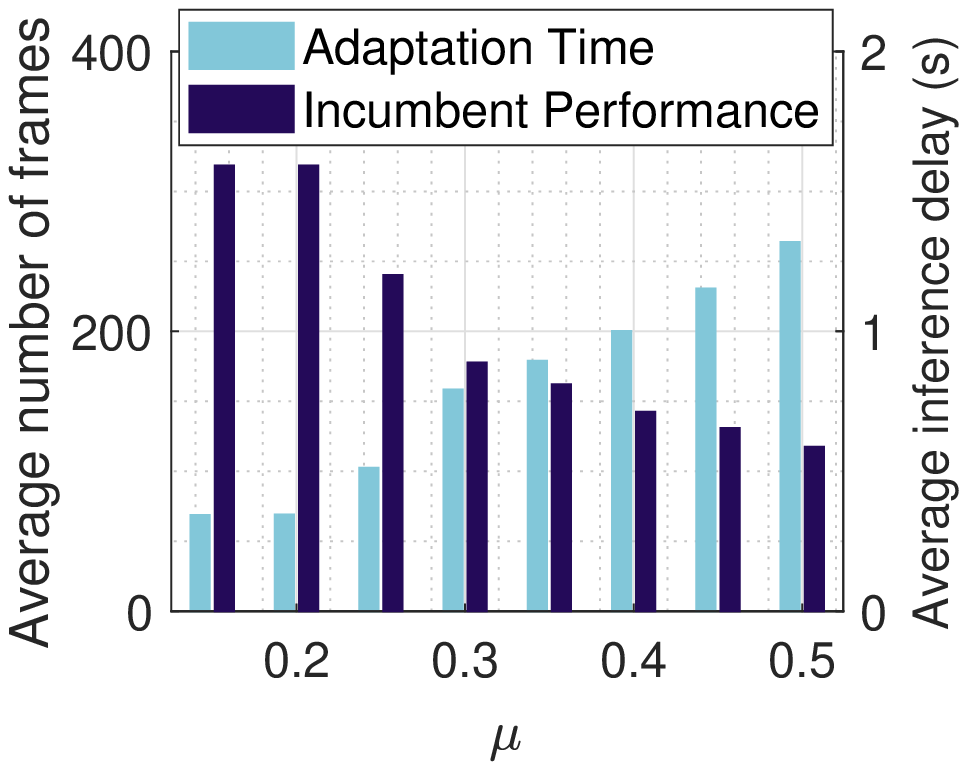}
\caption{Tradeoff of the forced sampling frequency. (larger $\mu$ means lower frequency.}
\label{fig:ForceImpact}
\end{minipage}
\vspace{-10 pt}
\end{figure}

\subsubsection*{\textbf{Forced Sampling Frequency}}
As shown in the last set of experiments, forced sampling is the key to continuing learning when pure on-device processing is selected for inference. In this experiment, we further investigate the impact of the forced sampling frequency on the learning performance. We design the experiment as follows: At frame $t_0$, the system starts with a bad network condition so that pure on-device processing is selected for inference. The system keeps in this bad network condition for a number frames, and then switches to a good network condition at frame $t_1$ where fc1 becomes the optimal DNN partition layer. At frame $t_2$, ANS stably selects the actual optimal layer fc1 as its partition decisions (i.e., for 20 consecutive frames). We vary the frequency of forced sampling to investigate its behavior with respect to the following two metrics: (1) \textbf{Adaptation Time}: Number of frames needed for ANS to detect the change and adapt its partition decision  i.e. $t_2 - t_1$. (2) \textbf{Incumbent Performance}: The end-to-end inference delay achieved by ANS during the period between $t_0$ and $t_1$. 

Clearly, there is an intuitive tradeoff between these two metrics. A higher forced sampling frequency allows ANS to quickly detect and adapt to the changes, thereby reducing the adaptation time. However, incumbent performance are harmed because of excessive unnecessary trials of suboptimal partition points. On the other hand, a lower forced sampling frequency minimizes the intervention to the incumbent optimal partition decision, namely on-device processing, but results in a longer time for ANS to detect changes due to less incoming new knowledge. The results in Fig. \ref{fig:ForceImpact} confirm this intuition.

\subsubsection*{\textbf{Key Frame Weights}}
\begin{figure}[t]
% 	\centering  
	\subfigure[]{
	    \label{fig:KeyFrameThreshold}
		\includegraphics[width=0.47\linewidth]{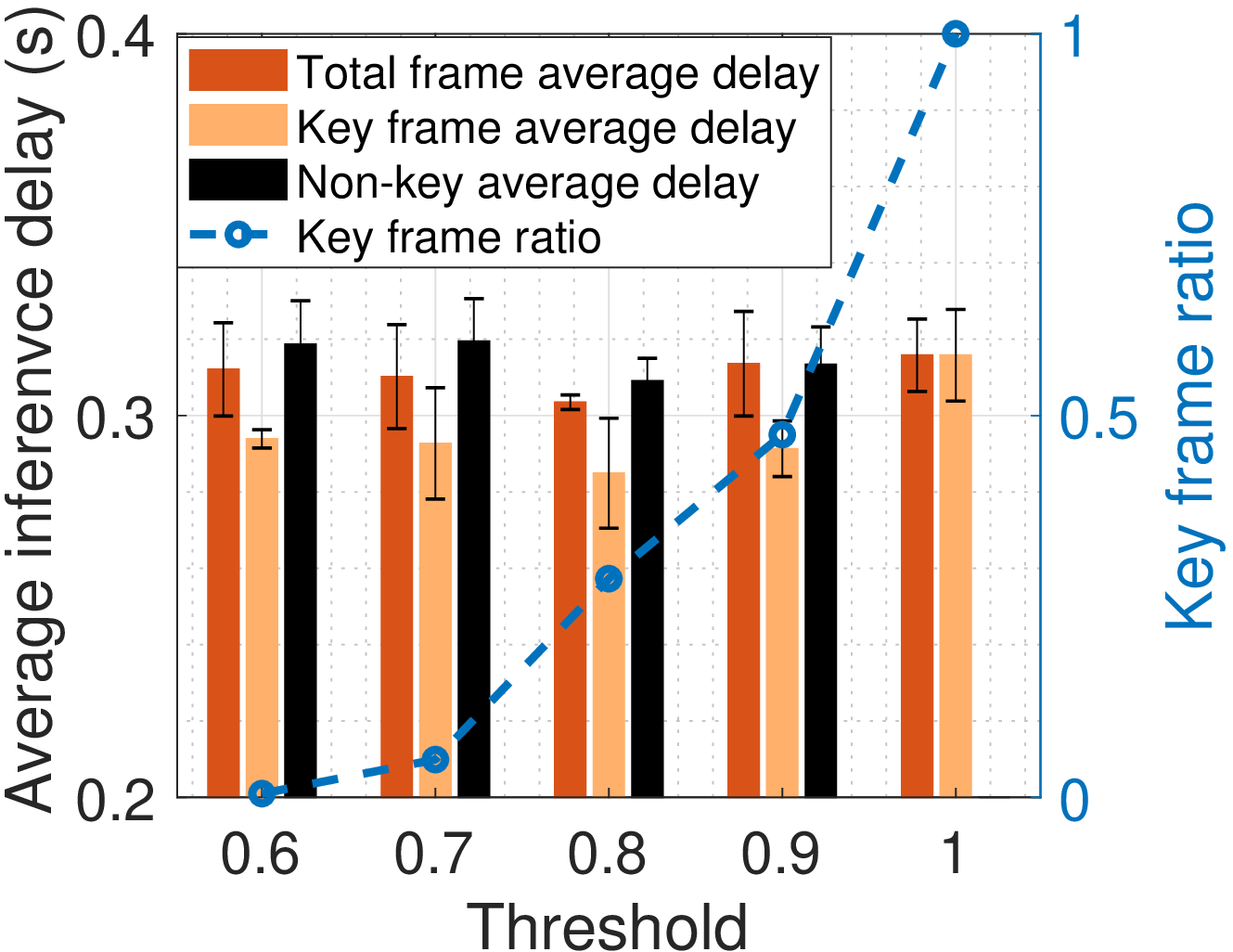}}
	\subfigure[]{
		\label{fig:ImpactweightL}
		\includegraphics[width=0.49\linewidth]{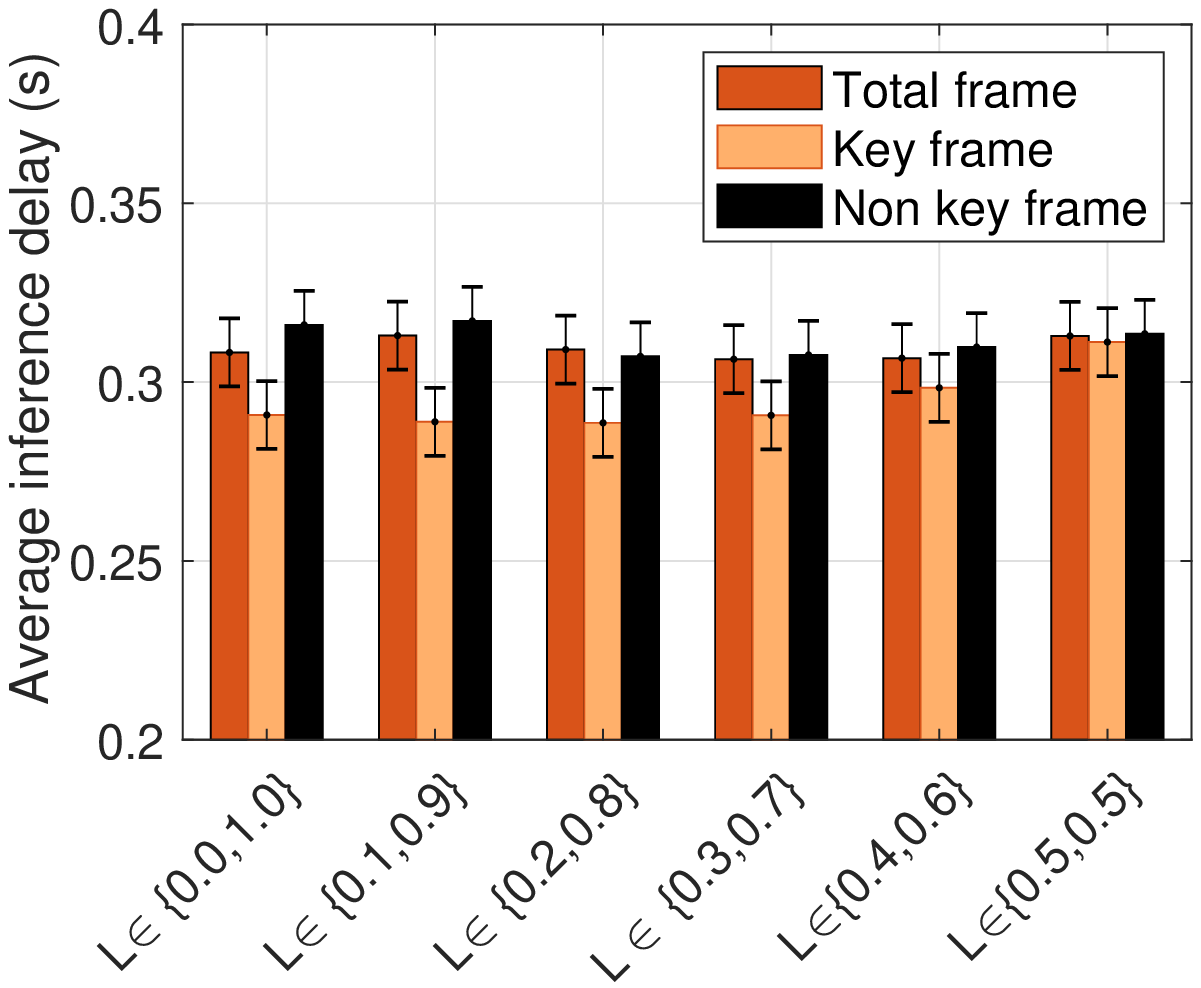}}
	\vspace{-10 pt}
	\caption{Inference delay of key frames: (a) Varying SSIM threshold for key frame detection. (b) Varying frame weights. }
	\vspace{-10 pt}
\end{figure}

In this experiment, we change the threshold adopted by the key frame detection algorithm SSIM to adjust the ratio of key frames, and investigate its impact on the end-to-end inference delay performance achieved by ANS. As shown in Fig. \ref{fig:KeyFrameThreshold}, for all detection thresholds, ANS is able to provide differentiated service to key and non-key frames, with the average delay of key frames considerably lower than that of the non-key frames. Note that when threshold is set to be 1, all frames are detected to be key-frames and hence, there are only two equal bars in the corresponding category. 

Furthermore, we change the relative weight $L_\text{key}/L_\text{non-key}$ assigned to key and non-key frames in ANS. As can be seen in Fig. \ref{fig:ImpactweightL}, when the relative weight becomes larger, the difference in the achieved end-to-end inference delay for key and non-key frames becomes more obvious. Therefore, by tuning the weights assigned to frames in ANS, we can provide desired differentiated services to different frames. 

\subsubsection*{\textbf{Working Together with Model Compression}}
\begin{figure}[tt]
\begin{minipage}[t]{0.46\linewidth}
\includegraphics[width=\textwidth]{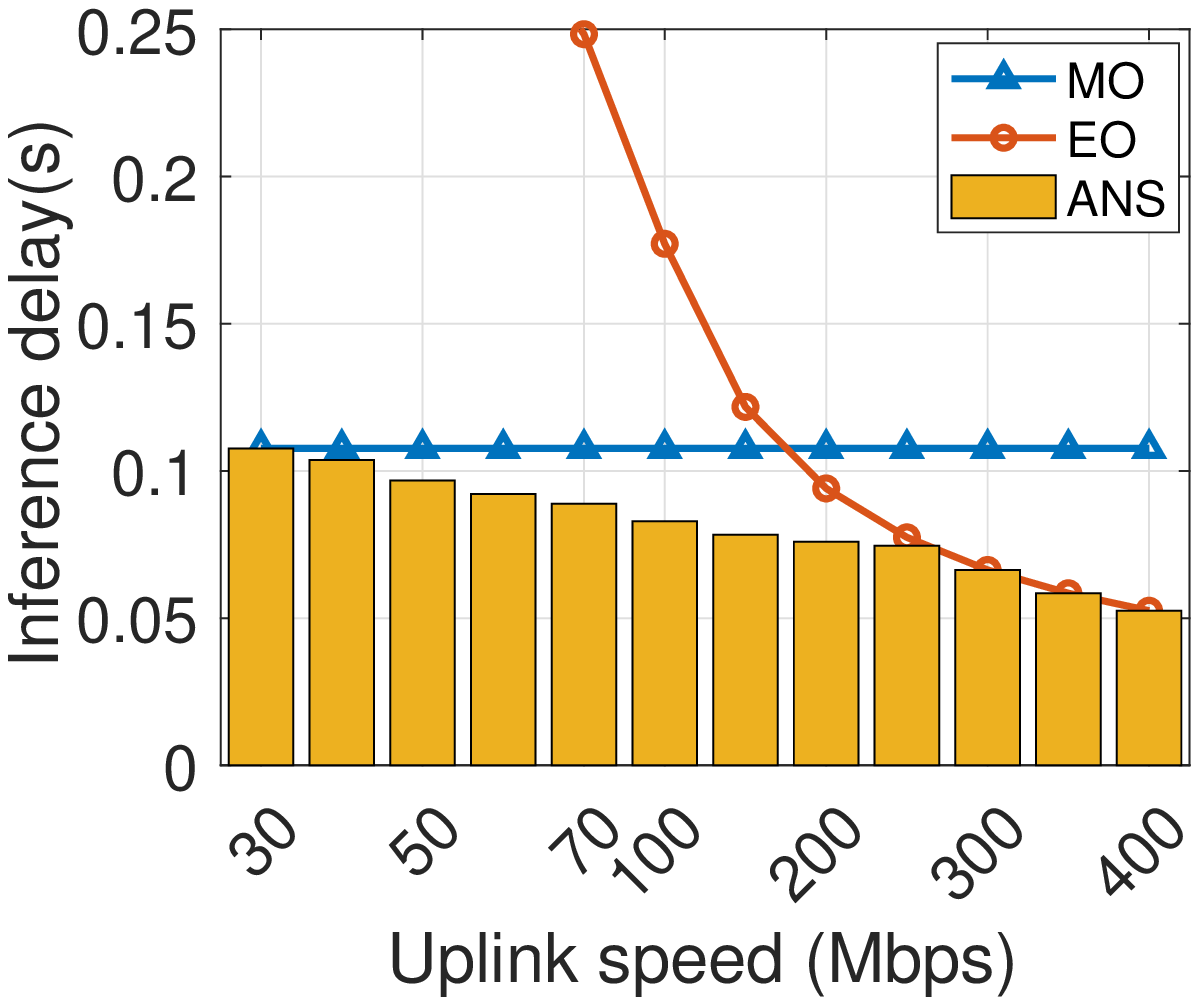}
\vspace{-10 pt}
\caption{ANS on compressed DNN YoLo-tiny.}
\label{fig:compressedNet}
\end{minipage}
\hspace{0.03\linewidth}
\begin{minipage}[t]{0.46\linewidth}
\includegraphics[width=\textwidth]{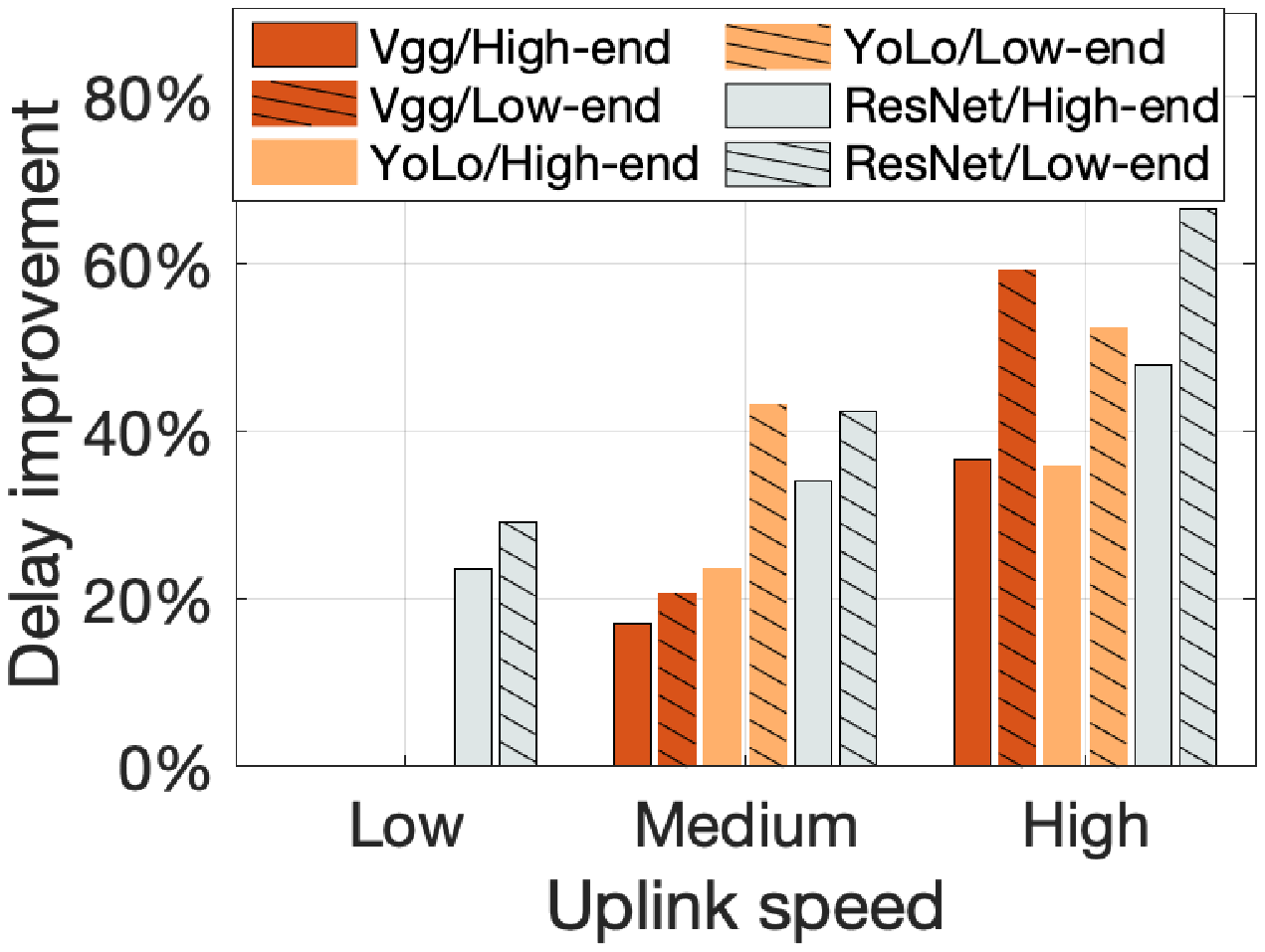}
\vspace{-10 pt}
\caption{ANS on different mobile devices.}
\label{fig:diffDevice}
\end{minipage}
\vspace{-10 pt}
\end{figure}

As highlighted in the Introduction, ANS does not aim to replace or compete with existing efforts such as DNN model compression. Rather, it is a natural complement to and enhancement of existing DNN model compression methods, which can work together without modification. To show this salient advantage of ANS, we design an experiment to let ANS work with a popular compressed DNN model for mobile devices, namely YoLo-tiny. The running time of YoLo-tiny is 7.76$\times$ less than the original uncompressed YoLo. Together with ANS, its inference delay is further reduced as shown in Fig. \ref{fig:compressedNet} for various uplink transmission rates, with the most significant gain obtained in the fast network regime (which is expected in the forthcoming mobile networks). This improvement is expected to be even larger on low-end mobile devices as will be shown in the next experiment. 

\subsubsection*{\textbf{Low-end Mobile Devices}}
As also mentioned in the Introduction, the current mobile device market is highly heterogeneous in terms of the devices' computing capability, with the vast majority being low-end and many years old. This is a real situation and obstacle that prevents DL-based mobile intelligence from being universally employed. In this experiment, we show that these low-end mobile devices can benefit even more from ANS, thereby enabling DL-based intelligence for a larger population. To this end, we configure the mobile device, namely Nvidia Jetson TX2, in two working modes using the \textit{nvpmodel} command. In the Max-N mode, the mobile device's GPU frequency is set as 1.30 GHz (referred to as the High-end mobile device); in the Max-Q mode, the mobile device's GPU frequency is set as 0.85 GHz (referred to as the Low-end mobile device), which is nearly halved compared to the first mode. We run ANS on these types of mobile devices for various DNN models in various network conditions, and report the delay reduction compared to MO (i.e. pure on-device processing) in Fig. \ref{fig:diffDevice}. As can been seen, the performance gain on low-end mobile devices is considerably higher than that of high-end mobile devices, especially with a fast network speed. Note that, when the network uplink speed is low, the delay reduction is 0 for Vgg16 and YoLo because ANS learns that the best decision is indeed to run all inference on the mobile device itself. In other cases, ANS learns a better partition point (possibly pure edge offloading) and thus improves the end-to-end inference delay performance.

\section{Related work}
\label{sec:RelatedWork}

\textbf{On-device Deep Learning Intelligence.} Deep learning for mobile and embedded devices has become a hot topic \cite{deng2019deep}, covering hardware architecture, computing platforms, and algorithmic optimization. Many CPU/GPU vendors are developing new processors to support tablets and smartphones to run DL empowered applications, a notable example being Apple Bionic chips \cite{sima2018apple}. New deep learning chips have already been developed on FPGAs \cite{guo2016model}. To support on-device inference, light versions of deep learning platforms have also emerged, e.g. TensorFlow Lite \cite{tfLite} and Caffe2 \cite{caffe2}. In addition, various algorithmic optimization techniques have been proposed.  For example, DNN compression methods \cite{xie2019source, zhang2018systematic} prune large-scale DNN models into computationally less demanding small DNNs that can be efficiently implemented on mobile devices; hardware acceleration methods \cite{nurvitadhi2017can, fowers2018configurable} are studied on the hardware level to optimize the utilization of hardware resources for accelerating DNN inference. However, these techniques are unlikely to address the immediate needs of all existing mobile devices, especially low-end and legacy devices that can not fully benefit from new computing architectures.  

\textbf{Multi-Access Edge Computing.} Multi-access edge computing \cite{taleb2017multi, shahzadi2017multi}, formally mobile edge computing, is an ETSI-defined network architecture concept that enables cloud computing capabilities and an IT service environment at the edge of the cellular network, and more in general at the edge of any network. By offloading data, running applications and performing related processing tasks closer to the end-users, network congestion is reduced and applications perform better. Some efforts have been made to migrate the DL related tasks to edge servers, e.g., \cite{zhou2018robust} builds a DL-based crowd sensing with edge computing. \cite{liu2017new} studies a DL-based food recognition system using edge computing, and \cite{liu2019edge} sets up an edge assisted real-time object detection for mobile augmented reality. By contrast, this paper does not simply migrate the DL service to the edge server but more importantly investigates an online learning-based DNN partitioning method to accelerate inference.

\textbf{Deep Neural Network Partition.} DNN partitioning technique is first proposed in the context of mobile cloud computing, where mobile device and cloud server work collaboratively to complete DNN inference tasks \cite{kang2017neurosurgeon,eshratifar2019jointdnn}. This technique is further extended to the edge computing scenario, e.g., \cite{li2019edge} studies a joint problem of DNN partition and early exit for edge computing systems and \cite{hu2019dynamic} proposes a partitioning scheme for DNNs with directed acyclic graph (DAG) topology. However, all these works require an offline profiling phase to measure the network condition, the processing ability of the mobile device, and the computing capacity of edge server. Given this information, these methods determine the optimal partition point by solving an optimization problem. However, the knowledge acquired during offline profiling can be easily outdated considering the highly dynamic environment in MEC systems and frequently updating the knowledge will incur large overhead. By contrast, our method uses online learning to learn the optimal partition on-the-fly.

\textbf{Contextual Bandit Learning.}
Multi-armed bandit (MAB) problem has been widely studied to address the critical tradeoff between exploration and exploitation in sequential decision making under uncertainty \cite{lai1985asymptotically}. Contextual bandit learning extends the basic MAB by considering the context-dependent reward functions to estimate the relation of the observed context and the uncertain environment, where LinUCB is a classic algorithm \cite{langford2007epoch}. 
AdaLinUCB \cite{guo2019adalinucb} builds on top of LinUCB to consider the heterogeneous importance of decision problems over time, which inspired our design for key frames. However, the special on-device processing decision causes a difficult challenge for
LinUCB and AdaLinUCB, forcing them to stop learning the first time when on-device processing is selected. Our algorithm $\mu$LinUCB incorporates a forced sampling technique to conquer this challenge, while still achieving provably asymptotically optimal performance.

\section{Conclusions} \label{sec:conclusion}
In this paper, we designed and prototyped a collaborative deep inference system to enable real-time object detection in video streams captured on mobile devices. We designed Autodidactic Neurosurgeon (ANS), an integral component of the system that online learns the optimal DNN partition points, using only limited delay feedback without a dedicated offline profiling/training phase. ANS explicitly considers varying importance of frames in video streams, and incorporates a simple yet effective forced sampling mechanism to ensure continued learning. As a result, ANS is able to closely follow the system changes and make adaptive decisions at a negligible computational cost. Experiments on our testbed show that ANS can significantly reduce the end-to-end inference delay compared to pure on-device processing or pure edge offloading through the synergy of both. Together with existing efforts on accelerating deep learning on resource constrained mobile devices such as DNN model compression, ANS will play an essential role in pushing the frontier of deep learning-empowered mobile intelligence, especially to the vast majority of low-end mobile devices. 

\bibliographystyle{IEEEtran}
\bibliography{reference}

% \newpage
\appendix
\section{Proof of Theorem \ref{theo:regertbound}} \label{sec:proofOfBound}
We analyze the performance of ANS by comparing it to an Oracle solution that knows precisely the ground-truth value of coefficients $\bm{\theta}^*$ and always picks the optimal partition point $p^*_t$ that minimizes the end-to-end inference delay for each frame $t$. The performance is measured in terms of the \emph{regret}, which is the gap of accumulated inference latency of all $T$ frames, $R = \sum_{t=1}^T d^\texttt{f}_{p_t} + {\bm{\theta}^*}^\top\x_{p_t} - d^\texttt{f}_{p^*_t} - {\bm{\theta}^*}^\top \x_{p^*_t}$. Before proving the main result \ref{theo:regertbound}, we first make some mild technical assumptions: (i) Noise $\eta$ satisfies the $C_\eta$-sub-Gaussian condition. (ii) The unknown parameter $\theta^*$ satisfies $\|\theta^*\|_2 \le C_{\theta}$. (iii) For $\forall p \in \mathcal{P}$, $\|x_p\|_2 \le C_{x}$ holds. (iv) The key frame weight $L_t$ satisfies $L_t \in \{L_\text{non-key},  L_\text{key}\}$, where $0 < L_\text{non-key} < L_\text{key} < 1$. (v) $\beta \ge \{1,  C_{\theta}^2\}$.

We classify the frames into three types: \textbf{Regular sampling sequence $\mathcal{R}$:} The frame is a normal frame and ANS selects an action in $\mathcal{P}$. \textbf{Non-sampling sequence $\mathcal{N}$:} The frame is a normal frame and ANS selects pure on-device processing $P$. \textbf{Forced sampling sequence $\mathcal{F}$:} The frame is a forced sampling frame and ANS selects an action in $\mathcal{P}_{\{\ne P\}}$. These frames are interspersed with each other as a result of ANS run. Let $\mathcal{T}_M = (t_1, \cdots, t_M)$ denote the subsequence of frames where each $t_m$ is a sampling frame (so ANS can observe $d^\texttt{e}_{p_t}$ and update the $\A_t$ and $\b_t$). Clearly, it must be $M \le T$. With abuse of notation, we use $\A_m$, $\b_m$ and $\bm{\theta}_m$ to denote the matrix, the vector and parameter estimation at the end of the $m$-th sampling frame thereafter.

\begin{lemma}
\label{pre_err}
(Prediction error bound) For any $\delta \in (0,1)$, with probability at least $1-\delta$, we have for all $p \in \mathcal{P}$ that
\begin{equation}
    |\hat{\bm{\theta}}_m^\top x_p - {\bm{\theta}^*}^\top x_p| \le \alpha \sqrt{(1-L_m)x^\top_p A_{m-1}^{-1} x_p}
\end{equation}
where $\alpha = \frac{C_\theta + C_\eta \sqrt{d\log\frac{1 + M C_x^2}{\delta}}}{1 - L_{key}}$.
\end{lemma}
\begin{proof}
{\allowdisplaybreaks
\begin{align*}
&|\hat{\bm{\theta}}_m^\top x_p - {\bm{\theta}^*}^\top x_p| = |(\hat{\bm{\theta}}_m^\top - {\bm{\theta}^*}^\top) x_p| = |(\hat{\bm{\theta}}_m^\top - {\bm{\theta}^*}^\top) \A^\frac{1}{2}_{m-1} \A^{-\frac{1}{2}}_{m-1} x_p| \notag\\
& = |(\hat{\bm{\theta}}_m^\top - {\bm{\theta}^*}^\top) \A^\frac{1}{2}_{m-1} x_p \A^{-\frac{1}{2}}_{m-1}| \le \|(\hat{\bm{\theta}}_m^\top - {\bm{\theta}^*}^\top) \A^\frac{1}{2}_{m-1}\|_2 \|x_p \A^{-\frac{1}{2}}_{m-1}\|_2 \notag\\
& = \sqrt{(\hat{\bm{\theta}}_m - \bm{\theta}^*)^\top \A^{\frac{1}{2}}_{m-1} \A^{\frac{1}{2}}_{m-1}(\hat{\bm{\theta}}_m - {\bm{\theta}^*})} \cdot \sqrt{x_p^\top \A^{-\frac{1}{2}}_{t-1} \A^{-\frac{1}{2}}_{t-1} x_p} \notag\\
& = \|\hat{\bm{\theta}}_m - \bm{\theta}^*\|_{\A_{m-1}} \cdot \sqrt{x_p^\top \A^{-1}_{m-1}x_p} \notag\\
& \le \left(C_\theta + C_\eta \sqrt{d\log\frac{1 + M C_x^2}{\delta}}\right) \cdot \sqrt{x_p^\top \A^{-1}_{m-1} x_p} \notag\\
& = \left(\frac{C_\theta + C_\eta \sqrt{d\log\frac{1 + M C_x^2}{\delta}}}{1 - L_m}\right) \cdot \sqrt{(1 - L_m)x_p^\top \A^{-1}_{m-1}x_p} \notag\\
& \le \left(\frac{C_\theta + C_\eta \sqrt{d\log\frac{1 + M C_x^2}{\delta}}}{1 - L_{key}}\right) \cdot \sqrt{(1 - L_m)x_p^\top \A^{-1}_{m-1}x_p}
\end{align*}
}
where the second equality holds by noting that $\A_{t-1}$ is a symmetric positive-definite matrix. The first inequality holds by the Cauchy-Schwarz inequality. The second inequality holds by Lemma \ref{sub_guassian} below. Thus, we set $\alpha = \frac{C_\theta + C_\eta \sqrt{d\log\frac{1 + M C_x^2}{\delta}}}{1 - L_{key}}$ and complete the proof.
\end{proof}

\begin{lemma}
\label{sub_guassian}
When $|\eta| \le C_\eta$, $\|\bm{\theta}^*\|_2 \le C_{\theta}$, $\|x_p\|_2 \le C_{x}$, for all $\delta \in (0,1)$, with probability at least $1-\delta$, we have
\begin{equation}
  \|\hat{\bm{\theta}}_m - \bm{\theta}^*\|_{\A_{m-1}} \le  C_\theta + C_\eta \sqrt{d\log\frac{1 + M C_x^2}{\delta}}  \notag
\end{equation}
where $d$ is the dimension of the context. 
\end{lemma}
\begin{proof}
The proof follows the fact that $\hat{\bm{\theta}}_t$ is the result of a ridge regression using data samples
collected in the sampling time slots , assuming the sub-Gaussian condition
for noise. For a complete proof, see Theorem 2 in \cite{abbasi2011improved}.
\end{proof}

\begin{lemma}
\label{oneStepRegret}
(One-step regret) $\forall m \ge 0$ Let $\alpha = \frac{C_\theta + C_\eta \sqrt{d\log\frac{1 + M C_x^2}{\delta}}}{1 - L_{key}}$, the one-step regret satisfies
\begin{align}
    & R_t \le 2 \alpha \sqrt{x^\top_p \A_m^{-1} x_p}, ~~if ~ t \in \mathcal{R} \\
    & R_t \le 3 \alpha \sqrt{x^\top_p \A_m^{-1} x_p}, ~~if ~ t \in \mathcal{N}
\end{align}
\end{lemma}
\begin{proof}
To prove the one-step regret, we note that when ANS chooses pure on-device processing ($p=P$), the latency of processing is $d^\texttt{f}_{P}$. Thus, we consider four cases in our algorithm and discuss the one-step regret.

(1) The optimal action is $p^* \in \{0, 1, \cdots, P-1\}$ and our algorithm selects action $p \in \{0, 1, \cdots, P-1\}$ (namely $t \in \mathcal{R}$). In this case, the one-step regret is
{\allowdisplaybreaks
\begin{align*}
& R_t = d^\texttt{f}_{p} + \bm{\theta}^{*\top} x_p - d^\texttt{f}_{p^*} - \bm{\theta}^{*\top} x_{p^*}  \notag \\
& \le d^\texttt{f}_{p} + \bm{\theta}^{*\top} x_p - d^\texttt{f}_{p^*} - \hat{\bm{\theta}}_m^\top x_{p^*} + \alpha \sqrt{(1-L_m) x^\top_{p^*} \A_{m-1}^{-1} x_{p^*}} \notag \\
& = d^\texttt{f}_{p} + \bm{\theta}^{*\top} x_p - [d^\texttt{f}_{p^*} + \hat{\bm{\theta}}_m^\top x_{p^*} - \alpha \sqrt{(1-L_m) x^\top_{p^*} \A_{m-1}^{-1} x_{p^*}}] \notag \\
& \le d^\texttt{f}_{p} + \bm{\theta}^{*\top} x_p - [d^\texttt{f}_p + \hat{\bm{\theta}}_m^\top x_p - \alpha \sqrt{(1-L_m) x^\top_p \A_{m-1}^{-1} x_p}]  \notag \\ 
& = \bm{\theta}^{*\top} x_p - \hat{\bm{\theta}}_m^\top x_p + \alpha \sqrt{(1- L_m)x^\top_p \A_{m-1}^{-1} x_p} \notag \\
& \le 2 \alpha \sqrt{(1- L_m)x^\top_p \A_{m-1}^{-1} x_p} \le 2 \alpha \sqrt{x^\top_p \A_{m-1}^{-1} x_p}
\end{align*}
}
where the inequalities in the second and sixth lines hold by Theorem \ref{pre_err}. The inequality in the fourth line holds by the design of our algorithm, especially by line 13.

(2) The optimal action is $p^* = P$ and our algorithm selects action $p = P$ (namely $t \in \mathcal{N}$). Thus, we have $R_t = 0$, because in this case $\mathbb{E}(d^\texttt{f}_{p^*} - d^\texttt{f}_{p}) = 0$.

(3) The optimal action is $p^* \in \{0, 1, \cdots, P-1\}$ and our algorithm selects action $p = P$ (namely $t \in \mathcal{N}$). We firstly introduce an auxiliary action $\hat{p} \in \{0, 1, \cdots, P-1\}$. Thus,
{\allowdisplaybreaks
\begin{align*}
& R_t = d^\texttt{f}_{P} - d^\texttt{f}_{p^*} - \bm{\theta}^{*\top} x_{p^*} \notag \\
& = d^\texttt{f}_{P} + d^\texttt{f}_{\hat{p}} + \bm{\theta}^{*\top} x_{\hat{p}} - d^\texttt{f}_{p^*} - \bm{\theta}^{*\top} x_{p^*} + \hat{\bm{\theta}}^\top_m x_{\hat{p}}  -  \bm{\theta}^{*\top} x_{\hat{p}} - d^\texttt{f}_{\hat{p}} - \hat{\bm{\theta}}_m x_{\hat{p}} \notag \\
& \le d^\texttt{f}_{P}  - d^\texttt{f}_{\hat{p}} - \hat{\bm{\theta}}_m x_{\hat{p}} + 3\alpha \sqrt{x^\top_{\hat{p}} A_{m-1}^{-1} x_{\hat{p}}} \le 3\alpha \sqrt{x_{\hat{p}}^\top \A_{m-1}^{-1} x_{\hat{p}}}
\end{align*}
}
where the inequality in the third line holds by the Lemma \ref{pre_err} and Case 1. The last inequality holds because $d^\texttt{f}_{P}  - d^\texttt{f}_{\hat{p}} - \hat{\bm{\theta}}_m x_{\hat{p}} \le 0$ in this case.

(4) The optimal action is $p^* = P$ and our algorithm selects action $p \in \{0, 1, \cdots, P-1\}$ (namely $t \in \mathcal{R}$). Thus,
\begin{equation}
\begin{aligned}
& R_t = d^\texttt{f}_{p} + \bm{\theta}^{*\top} x_p - d^\texttt{f}_{P} = (\bm{\theta}^{*\top} x_p - \hat{\bm{\theta}}_m^\top x_p) + (\hat{\bm{\theta}}_m^\top x_p + d^\texttt{f}_{p} - d^\texttt{f}_{P}) \notag \\
& \le 2 \alpha \sqrt{x^\top_p \A_{m-1}^{-1} x_p}
\end{aligned}
\end{equation}
where the inequality holds by Lemma \ref{pre_err}.

According to the discussion above, the one-step regret satisfies $R_t \le 3 \alpha \sqrt{x^\top_p \A_m^{-1} x_p}$.
\end{proof}

\begin{lemma}
\label{double_x}
Assume $\|x_p\|_2 \le C_{x}$ and the minimum eigenvalue of $\A_0$ satisfies $\lambda_{min}(\A_0) \ge max\{1, C_x^2\}$. Then, we have
\begin{equation}
    \sum_{t=1}^T x^\top_p \A_{t-1}^{-1} x_p \le 2\log(\frac{\det(A_M)}{\det{I_d}}) \le 2d[\log (\beta + \frac{M C_x^2}{d}) - \log \beta] \notag
\end{equation}
\end{lemma}
\begin{proof}
The proof follows Lemma 11 of \cite{abbasi2011improved}.
\end{proof}

With Lemma \ref{oneStepRegret} and Lemma \ref{double_x}, we can get the total regret incurred in the regular sampling sequence $\mathcal{R}$
{\allowdisplaybreaks
\begin{align*}
    & R_{\mathcal{R}} = \sum_{t=1}^T R_{t}\textbf{1}\{t \in \mathcal{R}\} \le \sqrt{M \sum_{m=1}^M R_{t_m}^2 \textbf{1}\{t \in \mathcal{R}\}} \notag\\ 
    & \le \sqrt{4M\alpha^2 \sum_{m=1}^M x^\top_{m, p} \A_m^{-1} x_{m, p}} \le 2\alpha \sqrt{2Md[\log (\beta + \frac{MC_x^2}{d}) - \log \beta]} \notag\\
    & \le 2\alpha \sqrt{2Td[\log (\beta + \frac{TC_x^2}{d}) - \log \beta]} = 2G(T)
\end{align*}
}
where the first inequality holds by the Jensen's inequality; the second inequality holds by Lemma \ref{oneStepRegret} and relaxing the indicator function $\textbf{1}\{t \in \mathcal{R}\}$; the third inequality holds by Lemma \ref{double_x}; the fourth inequality holds by the fact $M \le T$.

Next, we consider the total regret incurred in the non-sampling sequence $\mathcal{N}$.
{\allowdisplaybreaks
\begin{align*}
    & R_{\mathcal{N}} = \sum_{t=1}^T R_{t}\textbf{1}\{t \in \mathcal{N}\} \le T^\mu \sum_{m=1}^M R_{m}  \le T^\mu 3 \alpha \sqrt{x^\top_p \A_t^{-1} x_p}  = 3T^\mu \cdot G(T)
\end{align*}
}

Next, we consider the total regret incurred in the forced sampling sequence $\mathcal{F}$.
\begin{equation}
    R_{\mathcal{F}} = \sum_{t=1}^T R_{t}\textbf{1}\{t \in \mathcal{F}\} \le T^{1-\mu} \bigtriangleup_{max}
\end{equation}
where $\bigtriangleup_{max}$ is the maximum latency gap between local processing and other ANS's actions. Thus, combining these regret bounds, we obtain
{\allowdisplaybreaks
\begin{align*}
    & R_{total} = R_{\mathcal{R}} + R_{\mathcal{N}} + R_{\mathcal{F}} \le (2 + 3T^\mu)G(T) + T^{1-\mu} \bigtriangleup_{max} \notag \\
    & G(T) = \frac{C_\theta + C_\eta \sqrt{d\log\frac{1 + T C_x^2}{\delta}}}{1 - L_{key}} \cdot \sqrt{2Td[\log (\beta + \frac{TC_x^2}{d}) - \log \beta]} \notag \\
\end{align*}
}
where $G(T) = O(T^{0.5} \log(T/\delta))$.
Thus, Theorem 1 shows that the regret bound of ANS is sublinear in $T$, or $\max\{O(T^{0.5 + \mu}\log (T/\delta)), O(T^{1-\mu})\}$ by choosing $\mu \in (0, 0.5)$.

\end{document}